\PassOptionsToPackage{hyperfootnotes=false}{hyperref}
\documentclass[letterpaper, 10 pt, conference]{ieeeconf}  

\IEEEoverridecommandlockouts                              

\usepackage[stable]{footmisc}
\usepackage[english]{babel}
\usepackage[utf8]{inputenc}
\usepackage{balance}
\usepackage{xspace}
\usepackage{bbm}
\usepackage{rotating}
\usepackage{url}
\usepackage{csquotes}
\usepackage{cite}
\usepackage{subfig}
\usepackage{paralist}
\usepackage{multirow}
\usepackage{multicol}
\usepackage{amsmath,amssymb,mathtools}
\usepackage{booktabs}
\usepackage{xfrac}
\usepackage{array}
\usepackage{algorithm}
\usepackage[noend]{algpseudocode}
\usepackage{textcomp}
\usepackage{siunitx}
\usepackage{microtype}
\usepackage[usenames,dvipsnames]{xcolor}
\usepackage{pgfplots}
\pgfplotsset{compat=newest,unit code/.code={\si{#1}},plot coordinates/math parser=false,grid style={lightgray}, ylabel right/.style={
        after end axis/.append code={
            \node [rotate=90, anchor=north] at (rel axis cs:1,0.5) {#1};
        }   
    }}
\usepgfplotslibrary{units,external,groupplots,fillbetween}
\usetikzlibrary{positioning,angles,quotes,patterns,shapes,spy, shapes.misc,backgrounds}
\tikzstyle{block} = [draw, rectangle, minimum height=2em, minimum width=5em]
\tikzstyle{addon} = [draw, rectangle, rounded corners]
\tikzstyle{pinstyle} = [pin edge={<-,thin,black}]
\tikzstyle{pinstyle2} = [pin edge={->,thin,black}]
\tikzstyle{mult} = [draw, isosceles triangle]
\tikzstyle{circ} = [draw, circle]
\tikzstyle{coord} = [coordinate]
\tikzstyle{circ2} = [draw, circle,minimum width=3pt, inner sep=0]
\tikzset{>=latex}
\tikzset{radiation/.style={{decorate,decoration={expanding
waves,angle=90,segment length=4pt}}}}
\newcommand{\asymcloud}[2][.1]{%
\begin{scope}[#2]
\pgftransformscale{#1}%
\pgfpathmoveto{\pgfpoint{261 pt}{115 pt}} 
  \pgfpathcurveto{\pgfqpoint{70 pt}{107 pt}}
                 {\pgfqpoint{137 pt}{291 pt}}
                 {\pgfqpoint{260 pt}{273 pt}} 
  \pgfpathcurveto{\pgfqpoint{78 pt}{382 pt}}
                 {\pgfqpoint{381 pt}{445 pt}}
                 {\pgfqpoint{412 pt}{410 pt}}
  \pgfpathcurveto{\pgfqpoint{577 pt}{587 pt}}
                 {\pgfqpoint{698 pt}{488 pt}}
                 {\pgfqpoint{685 pt}{366 pt}}
  \pgfpathcurveto{\pgfqpoint{840 pt}{192 pt}}
                 {\pgfqpoint{610 pt}{157 pt}}
                 {\pgfqpoint{610 pt}{157 pt}}
  \pgfpathcurveto{\pgfqpoint{531 pt}{39 pt}}
                 {\pgfqpoint{298 pt}{51 pt}}
                 {\pgfqpoint{261 pt}{115 pt}}
\pgfusepath{fill}         
\end{scope}}    

\newenvironment{renum}
  {\begin{enumerate}}
  {\end{enumerate}}

\graphicspath{{./figures/}}

\usepackage{scalerel,stackengine}

\makeatletter
\newcommand{\myitem}[1]{%
\item[#1]\protected@edef\@currentlabel{#1}%
}
\makeatother

\newtheorem{theo}{Theorem}
\newtheorem{lem}{Lemma}

\newtheorem{cor}{Corollary}
\newtheorem{remark}{Remark}
\newtheorem{assume}{Assumption}
\newcommand{\fakepar}[1]{\vspace{1mm}\noindent\textbf{#1.}}

\DeclareSIUnit{\belmilliwatt}{Bm}
\DeclareSIUnit{\dBm}{\deci\belmilliwatt}

\newcommand{\norm}[1]{\left\lVert#1\right\rVert_1}
\newcommand{\abs}[1]{\left\lvert#1\right\rvert}
\newcommand*\diff{\mathop{}\!\mathrm{d}}
\DeclareMathOperator*{\R}{\mathbb{R}}

\newcommand{\transp}{\text{T}}

\DeclareMathOperator*{\argmax}{arg\,max}

\newcommand{\traj}[3]{\xi_{(#1,#2,#3)}}
\newcommand{\x}[1]{\tilde{x}_{#1}}
\DeclareMathOperator{\Rcon}{\ensuremath{R_\epsilon^\mathrm{c}}}
\DeclareMathOperator{\Rglob}{\ensuremath{R_\epsilon}}

\let\originalleft\left
\let\originalright\right
\renewcommand{\left}{\mathopen{}\mathclose\bgroup\originalleft}
\renewcommand{\right}{\aftergroup\egroup\originalright}

\newcommand\figref[1]{Fig.~\ref{#1}}
\newcommand\tabref[1]{Table~\ref{#1}}
\newcommand\secref[1]{Sec.~\ref{#1}}

\newcommand{\eg}{e.g.,\xspace}
\newcommand{\ie}{i.e.,\xspace}

\newcommand{\cf}{cf.\@\xspace}
\newcommand{\capt}[1]{\mdseries{\emph{#1}}}

\usepackage{ifthen}
\newboolean{authnotes}

\setboolean{authnotes}{true}

\ifthenelse{\boolean{authnotes}}
{
\newcommand{\am}[1]{\footnote{{\bf\color{blue!70!black} Alon: #1}}}
\newcommand{\db}[1]{\footnote{{\bf\color{green!50!black} Dominik: #1}}}
\newcommand{\st}[1]{\footnote{{\bf\color{purple!90!black} Sebastian: #1}}}
\newcommand{\mt}[1]{\footnote{{\bf\color{orange!50!black} Matteo: #1}}}
}
{
\newcommand{\am}[1]{}
\newcommand{\db}[1]{}
\newcommand{\st}[1]{}
\newcommand{\mt}[1]{}
}

\newcommand{\safeopt}{\textsc{SafeOpt}\xspace}
\newcommand{\ourmethod}{\textsc{GoSafe}\xspace}

\title{\LARGE \bf
GoSafe: Globally Optimal Safe Robot Learning
}

\author{Dominik Baumann$^{1,2}$, Alonso Marco$^{2}$, Matteo Turchetta$^{3}$, and Sebastian Trimpe$^{1,2}$
\thanks{This work was supported in part by the German Research Foundation (DFG) within SPP 1914 (grant TR 1433/1-1), the Cyber Valley Initiative, and the Max Planck Society.}
\thanks{$^{1}$Institute for Data Science in Mechanical Engineering, RWTH Aachen University, Germany
        {\tt\small \{dominik.baumann,trimpe\}@dsme.rwth-aachen.de}}%
\thanks{$^{2}$Max Planck Institute for Intelligent Systems, Stuttgart, Germany
        {\tt\small amarco@tuebingen.mpg.de}}%
\thanks{$^{3}$Department of Computer Science, ETH Z\"urich, Switzerland
        {\tt\small matteo.turchetta@inf.ethz.ch}}%
}

\usepackage{fancyhdr}
\newcommand{\mytitle}{\textbf{Accepted final version.}
To appear in \textit{Proc.\ of the International Conference on Robotics and Automation}. The main body is identical with the accepted final version (except for a minor correction and a minor clarification) while the appendix of this online version contains extended proofs of the main theorems.\\
\copyright 2021 IEEE. Personal use of this material is permitted. Permission
from IEEE must be obtained for all other uses, in any current or future
media, including reprinting/republishing this material for advertising or
promotional purposes, creating new collective works, for resale or
redistribution to servers or lists, or reuse of any copyrighted component of
this work in other works.}
\begin{document}
\bstctlcite{IEEEexample:BSTcontrol}

\maketitle
\thispagestyle{fancy}	
\pagestyle{empty}


\begin{abstract}
    When learning policies for robotic systems from data, safety is a major concern, as violation of safety constraints may cause hardware damage.
    \safeopt is an efficient Bayesian optimization (BO) algorithm that can learn policies while guaranteeing safety with high probability.
    However, its search space is limited to an initially given safe region.
    We extend this method by exploring outside the initial safe area while still guaranteeing safety with high probability.
    This is achieved by learning a set of initial conditions from which we can recover safely using a learned backup controller in case of a potential failure.
We derive conditions for guaranteed convergence to the global optimum and validate \ourmethod in hardware experiments.
\end{abstract}


\section{Introduction}
\label{sec:intro}
Algorithms that enable robotic systems to follow a trajectory or balance typically rely on a mathematical description of their dynamics through a model. 
Obtaining such a model is getting harder as robotic systems become more complex.
To mitigate the need for a dynamics model, model-free machine learning methods aim to directly learn policies from data.
However, such approaches require sufficient exploration, which can, due to the unknown dynamics, lead to failures, \ie violation of safety constraints.
In robotic systems, this should be avoided as it may cause hardware damage.

\safeopt ~\cite{berkenkamp2016safe} is a model-free algorithm that,
starting from a safe, albeit sub-optimal policy, explores new ones to improve the robot's performance while avoiding failures with high probability.
Crucially, it can only explore safe regions of the policy space connected to the initial policy.
Thus, if multiple disjoint safe areas exist, \safeopt cannot detect them and may miss the globally optimal safe policy (see \figref{fig:ill_example}).
Disjoint safe areas occur, \eg in parameterized linear systems~\cite{gryazina2006stability} and local optima can also be encountered when learning gaits of a bipedal robot~\cite{calandra2016bayesian}.

\looseness=-1
\safeopt approaches the policy search problem as a black-box optimization: it suggests a policy and receives information about its reward and safety without considering their generative process. 
However, when learning policies on dynamical systems like robots, we can monitor  the evolution of the system's state, and intervene in the experiment execution if there is an imminent danger. 
In this case, we can trigger a safe backup policy to ensure that all safety constraints are met. 
This allows us to evaluate potentially \enquote{unsafe} policies without violating any safety constraint. 
While \safeopt forgoes this possibility, \ourmethod exploits it to enable the exploration of disconnected safe regions in policy space, which, in turn, allows us to find globally optimal safe policies. 
To this end, we extend the definition of the safe set, which in the original \safeopt work is limited to the policy space, to a new set that joins policy and state space.
The projection of this set onto the policy space determines a set of provably safe backup policies, while its projection onto the state space determines the corresponding states from which a backup policy should be triggered. 
Thus, this set enables us to search globally for the optimal policy while preserving safety.



\begin{figure}
\centering
         
\begin{tikzpicture}
\begin{axis}
[
width=0.45\textwidth,
height=0.2\textheight,
enlargelimits=false,
ylabel={\textcolor{blue}{Reward}},
ylabel right={\textcolor{red}{Constraint}},
xlabel=Policy parameter,
ticks=none,
ymax=1, ymin=0
]
\addplot[domain=-1:1, samples=100, name path=reward, blue]{-2.26666667*x^4 + -0.26666667*x^3 + 1.76666667*x^2 + 0.26666667*x + 0.5};
\addplot[domain=-1:1, samples=100, name path=constraint, red]{19.3137*x^10 - 1.75115*x^9 - 34.2007*x^8 + 4.18593*x^7 + 19.9015*x^6 - 2.74682*x^5 - 7.90079*x^4 + 0.320741*x^3 + 2.78633*x^2 - 0.00869814*x + 0.3};
\addplot[domain=-1:1, dashed, very thick, red, name path=threshold]{0.6};
\path[name path=middle] (0,0) -- (0,1);
\draw[stealth-,thick](axis cs: -0.6,0.65)--(axis
cs:-0.4,0.3)node[below]{Initial safe area};
\draw[stealth-,thick](axis cs: 0,0.62)--(axis
cs:0,0.8)node[above]{Safety threshold};
\draw[stealth-,thick](axis cs: 0.55,0.65)--(axis
cs:0.35,0.15)node[below]{Additional safe area};

\path [name intersections={of=constraint and threshold}]; 
\coordinate (OP1) at (intersection-1);
\coordinate (OP2) at (intersection-2);
\coordinate (OP3) at (intersection-3);
\coordinate (OP4) at (intersection-4);

\begin{scope}[on background layer]
\draw[pattern=north west lines, pattern color=green!80!black, draw=none] (OP1|-0,0.6) rectangle (OP2|-0, 1);
\draw[pattern=north west lines, pattern color=green!80!black, draw=none] (OP3|-0,0.6) rectangle (OP4|-0, 1);
\end{scope}

\addplot+[blue,
  mark=x,
  only marks,
  mark size=5pt,
  mark options={line width=2pt},
  mark color=blue
] 
  coordinates
  {(0.617788, 0.94596)};


\addplot+[blue,
  mark=x,
  only marks,
  mark size=5pt,
  mark options={line width=2pt},
  mark color=blue
] 
  coordinates
  {(-0.63112, 0.742809)};

        
        
        
\end{axis}
\end{tikzpicture}
\caption{Illustrative example. \capt{If \safeopt is initialized in the left region, it will only be able to find the local optimum in this region but miss the global optimum.}}
\label{fig:ill_example}
\end{figure}
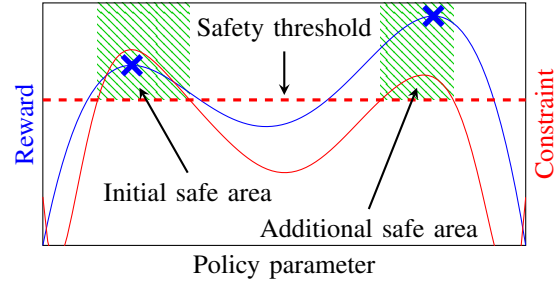

\fakepar{Contributions}
\looseness=-1
We propose \ourmethod, an extension of \safeopt that finds globally optimal safe policies for dynamical systems while fulfilling safety constraints at all times. 
Central to our approach is learning of a set of safe backup policies and states from which these should be triggered, \ie we essentially learn a safely operable region, which is useful in its own right.
We then leverage the obtained knowledge of this safely operable region to evaluate potentially unsafe policies and, therefore, discover better policies than those \safeopt is limited to.
Under appropriate assumptions on the system dynamics, we can guarantee convergence to the global optimum while avoiding failures with high probability.
The performance of \ourmethod is demonstrated in experiments on a Furuta pendulum\footnote{Video available at \url{https://youtu.be/YgTEFE_ZOkc}}.

\fakepar{Related work}
Learning control policies directly on robotic hardware requires efficient algorithms as data collection is time-consuming and causes hardware deterioration.
Bayesian optimization (BO)~\cite{mockus1978application} denotes a class of sample-efficient, black-box optimization algorithms that have, among others, been successfully applied to policy search on hardware platforms~\cite{antonova2017deep,calandra2016bayesian,marco2016automatic,turchetta2019robust}.
Based on BO, several algorithms have been proposed, which, apart from learning a control policy, also guarantee safety with high probability~\cite{sui2015safe,berkenkamp2016safe,berkenkamp2016bayesian,schreiter2015safe,schillinger2018safe}.
As discussed above, those methods are restricted to an initially given safe region and may miss the global optimum in case of disjoint safe regions (\cf \figref{fig:ill_example}).
Herein, we propose an extension of this line of work to also explore outside the initial safe area in hope of finding better optima, but without incurring in extra failures with high probability.
The above methods learn exclusively in parameter space.
Methods that also consider the state space, as we do to enable global exploration, have also been proposed.
However, they are based on a system model (given or learned from data)~\cite{berkenkamp2017safe,berkenkamp2016roa,akametalu2014reachability,turchetta2016safe,turchetta2019safe}.
In contrast, we present a \emph{model-free} approach.

Other methods for learning optimal policies while avoiding failures have been proposed in~\cite{heim2019learnable} leveraging viability theory and in~\cite{achiam2017constrained,chow2018lyapunov,garcia2015comprehensive}, which are based on reinforcement learning (RL).
While~\cite{heim2019learnable} cannot guarantee safety during exploration, RL methods are generally difficult to apply to real systems due to their sample inefficiency.
Another related line of work is Bayesian optimization under unknown constraints (BOC)~\cite{hernandez2016general,Gelbart2014,gardner2014bayesian,gramacy2011opti,schonlau1998global,picheny2014stepwise,marco2021robot}, which seeks to optimize an objective function subject to multiple constraints.
However, BOC assumes that failures come at no cost, thus incurring many failures during the search.
Recently,~\cite{marco2020excursion} has proposed a middle-ground solution in which exploration outside the initial safe area is enabled, but only under a pre-established limited number of failures. 
Contrary to these works, we consider a setting where zero failures are allowed.

\section{Problem Setting}
\label{sec:problem}

Assume an unknown, Lipschitz-continuous system
\begin{equation}
\label{eqn:sysdyn}
\diff x(t) = z(x(t),u(t))\diff t,
\end{equation}
where the control input $u(t) \in \R^{m}$ is supposed to drive the system state $x(t) \in\mathcal{X}\subset \R^{\ell}$ to some desired, possibly time-varying, state $x_\mathrm{des}(t)$ 
from a predefined initial state $x(0) = x_0$.
The control inputs are computed according to a parameterized policy $u(t) = \pi(x(t);a)$, where $a$ are the policy parameters $a\in\mathcal{A}\subseteq \R^d$. 
The quality of a policy $a$ when applied from the initial condition $x_0$ is quantified through the unknown reward function $f:\,\mathcal{A} \times \mathcal{X} \mapsto \R$.

Our goal is to learn the optimal policy parameters $a$ for a given initial condition $x_0$ while ensuring safety throughout the learning process. We assume that safety is encoded by a set of constraints on the current state of the system, \eg on the distance to an obstacle. Such constraint functions are expressed through the immediate constraint function $\bar{g}_i:\,\mathcal{X}\mapsto \R,\, i\in\{1,\ldots,q\}$.
The state trajectory is uniquely determined by $(a,x_0)$ through~\eqref{eqn:sysdyn}.
Thus, we further define $g_i:\,\mathcal{A}\times\mathcal{X}\mapsto\R,\, i\in\{1,\ldots,q\}$ as $g_i(a,x_0) = \min_{t\geq 0} \bar{g}_i(x(t))$, \ie the minimum value observed throughout a trajectory starting from $x_0$ with parameters $a$.
Also $\bar{g}_i$ and $g_i$ are unknown a priori but can be estimated through experiments on the real system.

The constrained optimization problem is
\begin{align}
\max_{a\in\mathcal{A}}f(a,x_0) \quad \text{subject to} \quad g_i(a,x_0)\ge 0 \;\forall i\in\{1,\ldots,q\},\nonumber
\end{align}
whose constraints must be satisfied at each time step.
Generally, we could also optimize over $\mathcal{X}$ to find initial conditions that yield a higher reward.
However, here we assume that the goal is to learn a policy from predefined initial conditions $x_0$, as is typically the case when letting a robot learn a certain behavior.
To guarantee safety throughout the learning process, we assume that an initial set $S_0$ of safe but possibly sub-optimal parameters for the nominal $x_0$ are given.
In robotics, such initial parameters can often be obtained from a simulation or domain knowledge.

The objective $f$ and the constraints $g_i$ are unknown a priori.
For estimating them, we assume that both have a bounded norm in a reproducing kernel Hilbert space (RKHS)~\cite{schlkopf2018learning} induced by a kernel $k$.
For this class of functions,~\cite{srinivas2012information,Chowdhury2017OnKM} have shown that Gaussian processes (GPs)~\cite{williams2006gaussian} can be used to obtain well-calibrated confidence intervals.
Informally, a GP is a Bayesian non-parametric regression tool that places a probabilistic prior over the unknown functions $f$ and $g_i$. 
This allows us to provide confidence intervals and to make statements about their Lipschitz continuity.

\section{Preliminaries}
\label{sec:background}

In this section, we introduce the necessary background on GPs, BO, and the \safeopt framework.

\fakepar{Gaussian processes}
A GP is a non-parametric model for a function $f$, fully defined by a \emph{kernel} and a mean function~\cite{williams2006gaussian}. 
The model is updated with new observations of $f$, which are assumed to be corrupted by white i.i.d.\ Gaussian noise, \ie $\hat{f}(\zeta) = f(\zeta) + v,\; v \sim \mathcal{N}(0,\sigma^2)$. 
We can then express the posterior distribution conditioned on those observations in closed form,
$\mu_n(\zeta^*)=k_n(\zeta^*)(K_n+I_n\sigma^2)^{-1}\hat{f}_n$ and
$\sigma^2_n(\zeta^*)=k(\zeta^*,\zeta^*)-k_n(\zeta^*)(K_n+I_n\sigma^2)^{-1}k_n^\transp(\zeta^*)$, respectively.
Here, $\hat{f}_n$ is the vector of observed function values, the entry $i,j$ of the covariance matrix $K_n\in\R^{n\times n}$ is $k(\zeta_i,\zeta_j)$, $k_n(\zeta^*)=[k(\zeta^*,\zeta_1),\ldots,k(\zeta^*,\zeta_n)]$ represents the covariance between the current evaluation $\zeta^*$ and the observed data points, and $I_n$ the $n\times n$ identity matrix.

This allows learning an approximation of the (scalar) reward function $f$.
To include the constraints $g_i$, we extend the GP framework with a surrogate selector function as in~\cite{berkenkamp2016bayesian}, 
\begin{align}
\label{eqn:gp_surrogate}
h(a,x_0,i) = \begin{cases}
f(a,x_0) & \text{ if } i=0\\
g_i(a,x_0) & \text{ if }i\in\mathcal{I}_g,
\end{cases}
\end{align}
where $i\in\mathcal{I}$ with $\mathcal{I} = \{0,\ldots,q\}$ and $\mathcal{I}_g=\{1,\ldots,q\}\subset\mathcal{I}$ (\ie the indices of the constraints) denotes whether the reward function or one of the constraints is returned.
When indexed for a particular $i$, the surrogate $h(\cdot)$ is also a GP that can be used to predict expectations and uncertainties of the reward and constraint functions.

\fakepar{Bayesian optimization}
BO adopts GP models to capture the belief about the objective function based on the available data.
The GP model's mean and variance are then used to suggest the next sample location in search of the objective function's optimum.
One such algorithm is GP-upper confidence bound (GP-UCB)~\cite{srinivas2012information}, which we will use in later sections, and which trades off exploration and exploitation by suggesting the next sample as
$
a_n = \argmax_{a\in\mathcal{A}}(\mu_{n-1}(a) + \beta_n^{\sfrac{1}{2}}\sigma_{n-1}(a)),$
where the scalar $\beta_n$ is iteration-dependent and reflects the confidence interval of the GP.
By iteratively evaluating the function at locations proposed by GP-UCB and updating the GP, GP-UCB provides theoretical bounds on the cumulative regret, which are ensured by choosing $\beta_n$ appropriately at each iteration~\cite{srinivas2012information}.

\fakepar{\safeopt}
\safeopt~\cite{sui2015safe} is a variant of GP-UCB that keeps track of a \emph{safe set}, i.e., a set of parameters that satisfy safety constraints with high probability, given the GP model. 
After each experiment, the safe set is updated.
During learning, \safeopt is trading off finding the optimum inside the currently known safe set with expanding the safe set.
Existing applications of \safeopt to the policy search problem~\cite{berkenkamp2016bayesian} only consider the parameter space $\mathcal{A}$ but ignore the state space $\mathcal{X}$.
In parameter space, \safeopt can guarantee to find the optimum connected to an initial safe set $S_0$ (up to $\epsilon$-precision) while guaranteeing safety with high probability.

\section{Globally Safe Optimization}

While \safeopt can provide guarantees on finding the optimum within an initial safe area, it cannot reach better optima that are not connected to this area.
Thus, the quality of the solution strongly depends on its initialization.
This is due to two main factors: (\textit{i}) \safeopt models safety as a function of the policy parameters only, rather than policy parameters \emph{and} initial conditions; (\textit{ii}) \safeopt assumes that the violation of a given constraint is instantaneous and there is no possibility to intervene once a controller is chosen.
In the following, we present \ourmethod, an extension of \safeopt  that exploits the additional structure of the policy search problem in robotics to enable exploration outside the initial safe area.
In particular, we extend the safe set by also including initial conditions and specifically search for initial conditions $x(0)=\tilde{x}_0$ from which we can still guarantee safety using some $a\in\mathcal{A}$.
This allows for intervening during experiments and switching to a safe backup policy in case of potential constraint violation.
We proceed by giving an intuition for \ourmethod, formally defining the new update rules, and stating its theoretical properties regarding safety and optimality.
Finally, we discuss some practical considerations.

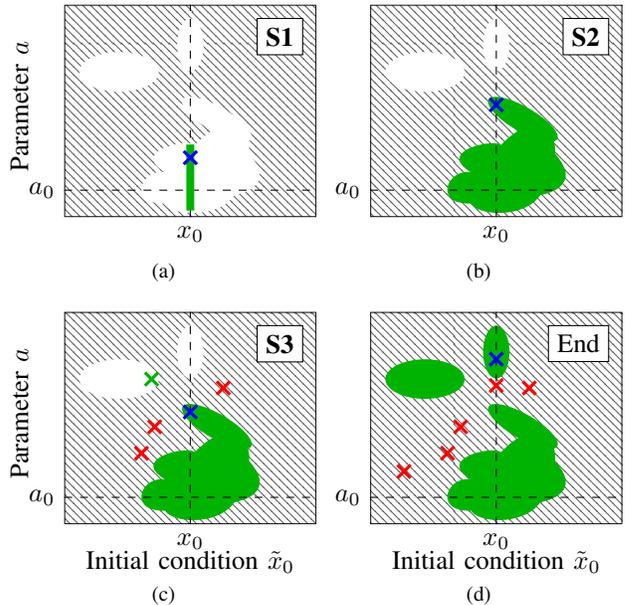
\begin{figure}
    \centering
    \tikzset{cross/.style={cross out, draw=blue, very thick, minimum size=0.5em, inner sep=0pt, outer sep=0pt},
cross/.default={1pt}}
\tikzset{cross2/.style={cross out, draw=red, very thick, minimum size=0.5em, inner sep=0pt, outer sep=0pt},
cross/.default={1pt}}
\tikzset{cross3/.style={cross out, draw=green!70!black, very thick, minimum size=0.5em, inner sep=0pt, outer sep=0pt},
cross/.default={1pt}}
\subfloat[\label{sfig:stage1}]{%
\begin{tikzpicture}
\node[draw, rectangle, minimum width=9.5em, minimum height=8em, pattern=north west lines, pattern color=darkgray, opacity=0.9](plot){};
\node at($(plot.south)+(0.5em, 2em)$)(safe1) {\tikz\asymcloud[.075]{fill=white};}; 
\node[draw=none, ellipse, minimum width=3em, minimum height=1.5em, below right = 2em and 1em of plot.north west, fill=white](safe2){};
\node[draw=none, ellipse, minimum width=3em, minimum height=0.5em, below right = 3.5em and 0.25em of plot.north, fill=white, rotate=-30](safe5){};
\node[draw=none, ellipse, minimum width=1em, minimum height=2em, below = 0.5em of plot.north, fill=white](safe6){};
\draw[dashed] (plot.south)node[below]{$x_0$} -- (plot.north);
\draw[dashed] ($(plot.west)+(0em, -3em)$)node[left]{$a_0$} -- ($(plot.east)+(0em, -3em)$);
\draw[green!70!black, line width=3pt]($(plot.south)+(0em,0.25em)$) -- ($(plot.south)+(0em,2.75em)$);
\node[draw, cross, above=2em of plot.south](opt){};
\node[above left=3em and 1em of plot.west, rotate=90]{Parameter $a$};
\node[draw, rectangle, fill=white, below left=0.5em and 0.5em of plot.north east]{\ref{item:S1}};
\end{tikzpicture}
}
\subfloat[\label{sfig:stage2}]{%
\begin{tikzpicture}
\node[draw, rectangle, minimum width=9.5em, minimum height=8em, pattern=north west lines, pattern color=darkgray, opacity=0.9](plot){};
\node at($(plot.south)+(0.5em, 2em)$)(safe1) {\tikz\asymcloud[.075]{fill=green!70!black};}; 
\node[draw=none, ellipse, minimum width=3em, minimum height=1.5em, below right = 2em and 1em of plot.north west, fill=white](safe2){};
\node[draw=none, ellipse, minimum width=3em, minimum height=0.5em, below right = 3.5em and 0.25em of plot.north, fill=green!70!black, rotate=-30](safe5){};
\node[draw=none, ellipse, minimum width=1em, minimum height=2em, below = 0.5em of plot.north, fill=white](safe6){};
\draw[dashed] (plot.south)node[below]{$x_0$} -- (plot.north);
\draw[dashed] ($(plot.west)+(0em, -3em)$)node[left]{$a_0$} -- ($(plot.east)+(0em, -3em)$);
\node[draw, cross, above=4em of plot.south](opt){};
\node[draw, rectangle, fill=white, below left=0.5em and 0.5em of plot.north east]{\ref{item:S2}};
\end{tikzpicture}
}

\subfloat[\label{sfig:stage3}]{%
\begin{tikzpicture}
\node[draw, rectangle, minimum width=9.5em, minimum height=8em, pattern=north west lines, pattern color=darkgray, opacity=0.9](plot){};
\node at($(plot.south)+(0.5em, 2em)$)(safe1) {\tikz\asymcloud[.075]{fill=green!70!black};}; 
\node[draw=none, ellipse, minimum width=3em, minimum height=1.5em, below right = 2em and 1em of plot.north west, fill=white](safe2){};
\node[draw=none, ellipse, minimum width=3em, minimum height=0.5em, below right = 3.5em and 0.25em of plot.north, fill=green!70!black, rotate=-30](safe5){};
\node[draw=none, ellipse, minimum width=1em, minimum height=2em, below = 0.5em of plot.north, fill=white](safe6){};
\draw[dashed] (plot.south)node[below]{$x_0$} -- (plot.north);
\draw[dashed] ($(plot.west)+(0em, -3em)$)node[left]{$a_0$} -- ($(plot.east)+(0em, -3em)$);
%
\node[draw, cross2, below left=0em and 1em of safe5](opt){};
\node[draw, cross2, below left = 1em and 1.5em of safe5]{};
\node[draw, cross2, above right = 0.5em and 0.5em of safe1.north]{};
\node[draw, cross, above=4em of plot.south](opt){};
\node[draw, cross3, right=2.5em of safe2.west](curr){};
\node[above left=3em and 1em of plot.west, rotate=90]{Parameter $a$};
\node[below=1em of plot, inner sep=0]{Initial condition $\tilde{x}_0$};
\node[draw, rectangle, fill=white, below left=0.5em and 0.5em of plot.north east]{\ref{item:S3}};
\end{tikzpicture}
}
\subfloat[\label{sfig:stage4}]{%
\begin{tikzpicture}
\node[draw, rectangle, minimum width=9.5em, minimum height=8em, pattern=north west lines, pattern color=darkgray, opacity=0.9](plot){};
\node at($(plot.south)+(0.5em, 2em)$)(safe1) {\tikz\asymcloud[.075]{fill=green!70!black};}; 
\node[draw=none, ellipse, minimum width=3em, minimum height=1.5em, below right = 2em and 1em of plot.north west, fill=green!70!black](safe2){};
\node[draw=none, ellipse, minimum width=3em, minimum height=0.5em, below right = 3.5em and 0.25em of plot.north, fill=green!70!black, rotate=-30](safe5){};
\node[draw=none, ellipse, minimum width=1em, minimum height=2em, below = 0.5em of plot.north, fill=green!70!black](safe6){};
\draw[dashed] (plot.south)node[below]{$x_0$} -- (plot.north);
\draw[dashed] ($(plot.west)+(0em, -3em)$)node[left]{$a_0$} -- ($(plot.east)+(0em, -3em)$);
%
\node[draw, cross2, left=0.5em of safe1](opt){};
\node[draw, cross2, below left=0em and 1em of safe5](opt){};
\node[draw, cross2, above=5em of plot.south]{};
\node[draw, cross2, below left = 1em and 1.5em of safe5]{};
\node[draw, cross2, above right = 0.5em and 0.5em of safe1.north]{};
\node[draw, cross, above=6em of plot.south](opt){};
\node[below=1em of plot, inner sep=0]{Initial condition $\tilde{x}_0$};
\node[draw, rectangle, fill=white, below left=0.5em and 0.5em of plot.north east]{End};
\end{tikzpicture}
}
    \caption{The different stages of \ourmethod. \capt{Gray hatched areas are unsafe, green areas are already explored, red crosses mark experiments where we needed to intervene, and the blue cross the current best guess. We first explore the parameter space of the initial safe area with the fixed nominal initial condition $x_0$ (\ref{item:S1}), then explore also in the initial condition space (\ref{item:S2}). Afterward, we search globally for new safe areas (\ref{item:S3}). In each new area, we revisit \ref{item:S1} and \ref{item:S2} such that in the end, all areas are fully explored. 
    }}
    \label{fig:stages_framework}
\end{figure}

\subsection{Global Optimization Framework}
\looseness=-1
We start with an intuitive explanation of \ourmethod and then make the required concepts precise.
\ourmethod proceeds in three stages (\cf \figref{fig:stages_framework}) of increasing complexity.
At each iteration, we check the stages' conditions in the order that follows below and execute the first whose condition is satisfied.
\begin{compactitem}
    \myitem{\textbf{S1}}\label{item:S1} Given the fixed nominal initial condition $x_0$, if we can expand the set of safe parameters or have a candidate for the optimal parameters, we do a \safeopt step in parameter space $\mathcal{A}$ as in~\cite{berkenkamp2016bayesian}.
    \myitem{\textbf{S2}}\label{item:S2} If we can further expand the safe set by considering pairs of initial conditions $\tilde{x}_0$ and parameters $a$ that are guaranteed to be safe, we do a safe exploration in this joint space.
    \myitem{\textbf{S3}}\label{item:S3} Else, we further expand the safe set by sampling $a$ from $\mathcal{A}$, \ie by doing a global search in parameter space, and only sample $\tilde{x}_0$ from the safe set.
    Here, we may need to intervene, \ie switch to a backup policy, during the experiment to preserve safety.
\end{compactitem}
\ref{item:S1} is the simplest as it involves searching over a reduced space (\ie only the vertical axis in \figref{sfig:stage1}) and only considers safe parameters. 
Its purpose is to find safe, optimal parameters for the nominal initial condition of the system. 
\ref{item:S2} is more complex as we enlarge the search space to include the initial condition. 
This stage is needed to determine a set of backup policies and a set of initial conditions from which these backup policies can guarantee safety. 
As in \figref{sfig:stage2}, this can reveal safe areas that were disconnected in parameter space and, thus, potentially better optima.
\ref{item:S3} is even more involved as the search space is enlarged further by considering potentially unsafe policy parameters $a\in\mathcal{A}$, which can be safely evaluated due to the backup policies estimated in \ref{item:S2}. 
Note that if we executed \ref{item:S3} before \ref{item:S2}, we would need to interrupt many experiments as the safe set would be very conservative.
The incentive of \ref{item:S3} is to find disconnected safe areas (\cf \figref{sfig:stage3}), which can then be further explored by revisiting \ref{item:S1} and \ref{item:S2}.
That way, we thoroughly explore all safe areas and, in the end, find the global optimum in \figref{sfig:stage4}.
In the following, we provide definitions of the required sets and detail the acquisition function for each of the three stages.
The whole framework is summarized in Alg.~\ref{alg:global_safeopt}.

\fakepar{Set definitions}
For defining confidence intervals and the required sets, we extend the definitions from~\cite{berkenkamp2016bayesian} to include the initial condition.
Following~\cite{berkenkamp2016bayesian}, we consider a finite set, $\mathcal{A}$, bounded $\mathcal{X}$, and introduce a quantized version of $\mathcal{X}$, $\mathcal{X}_\mu$ such that $\norm{x(t)-[x(t)]_\mu}\le\mu\,\forall x\in\mathcal{X}$ and $[x(t)]_\mu$ the element in $\mathcal{X}_\mu$ with the smallest $L_1$-distance to $x(t)$. 
Naturally, we then need that if $(a,x_0)\in S_0\subseteq\mathcal{A}\times\mathcal{X}_\mu$ also all $x$ for which $\norm{x-x_0}\le\mu$ are safe, \ie we assume $\forall (a, x_0)\in S_0: g_i(a, x_0)>L_\mathrm{x}\mu\,~\forall i\in\mathcal{I}_\mathrm{g}$.
Further, we assume that $f$ and $g_i$ are Lipschitz-continuous with Lipschitz constants $L_\mathrm{a}$ and $L_\mathrm{x}$.
We construct confidence intervals for the surrogate function~\eqref{eqn:gp_surrogate} as
$
Q_n(a,\tilde{x}_0,i)=[\mu_{n-1}(a,\tilde{x}_0,i)\pm \beta_n^{\sfrac{1}{2}}\sigma_{n-1}(a,\tilde{x}_0,i)].
$
\safeopt requires that the safe set does not shrink.
This can be ensured by defining the contained set as $C_n(a,\tilde{x}_0,i) = C_{n-1}(a,\tilde{x}_0,i)\cap Q_n(a,\tilde{x}_0,i)$, where $C_0(a,\tilde{x}_0,i)$ is $[L_\mathrm{x}\mu,\infty)\,\forall (a,\tilde{x}_0)\in S_0$ and $\R$ otherwise.
Then, upper and lower bounds can be defined as $l_n(a,\tilde{x}_0,i)\coloneqq\min C_n(a,\tilde{x}_0,i)$ and $u_n(a,\tilde{x}_0,i)\coloneqq\max C_n(a,\tilde{x}_0,i)$.
Finally, we update the safe set with\footnote{\label{note:changes}\fakepar{Corrections} Different from the ICRA 2021 paper, we here have a union with $S_{n-1}$. This is to ensure that points with $0\le l_n(a_n,\x0,i)\le L_\mathrm{x}\mu$ are not removed from the safe set.
Further, we do not set $l_n(a,\tilde{x}_0,i) = \max\{L_\mathrm{x}\mu,l_n(a,\tilde{x}_0,i)\}$ for points added during \textbf{S3} anymore, correcting a minor error in the ICRA version.}
\begin{align}
\label{eqn:safeset_update}
\begin{split}
S_n &= \bigcap\limits_{i\in\mathcal{I}_g}\bigcup\limits_{(a,\tilde{x}_0) \in S_{n-1}} \{a'\in\mathcal{A},\tilde{x}_0'\in\mathcal{X}_\mu\mid 
l_n(a,\tilde{x}_0,i)-
\\&L_\mathrm{a}\norm{a-a'}-L_\mathrm{x}(\norm{\tilde{x}_0-\tilde{x}_0'}+\mu)\ge 0\} \cup S_{n-1}.
\end{split}
\end{align}
As in~\cite{berkenkamp2016bayesian}, we define two subsets of $S_n$ to find a trade-off between expanding the safe set and finding the optimal parameters inside the current safe set.
These subsets contain parameters that are likely to improve the current estimate of the maximum (potential maximizers, $M_n$) or are likely to expand the safe set (potential expanders, $G_n$) and are formally defined as
$M_n\coloneqq\{(a,x_0)\in S_n\mid u_n(a,x_0,0)\ge\max_{(a',\x0=x_0)\in S_n}l_n(a',x_0,0)\}$ and 
$G_n\coloneqq\{(a,\tilde{x}_0)\in S_n\mid e_n(a,\tilde{x}_0)>0\},$
with $e_n(a,\tilde{x}_0) = \lvert\{(a',\tilde{x}_0')\in\mathcal{A}\times\mathcal{X}_\mu\setminus S_n\mid\exists i\in\mathcal{I}_g:u_n(a,\tilde{x}_0,i)-L_\mathrm{a}\norm{a-a'}-L_\mathrm{x}(\norm{\tilde{x}_0-\tilde{x}_0'}+\mu)\ge 0\}\rvert.$
Here, $e_n(a,\tilde{x}_0)$ denotes the number of $(a, \tilde{x}_0)$ pairs which are guaranteed to satisfy at least one constraint given an optimistic observation.
Note that for $M_n$, we consider $x_0$ instead of $\tilde{x}_0$, as we assume that the optimal policy parameters are to be found for the nominal $x_0$.
For knowing when to intervene in an experiment, we further need to define the border of the safe set, 
$\partial S_n\coloneqq \{\x0\in\mathcal{X}_\mu\mid \exists a\in\mathcal{A}: (a,\x0)\in S_n \land (\exists x\in\mathcal{X}:\norm{\x0-x}< 2\mu \land \nexists a\in\mathcal{A}:(a,[x]_\mu)\in S_n)\}$.
Intuitively, it defines the set of discrete states for which at least one neighbour does not belong to the safe set.
Lastly, we define a set $\mathcal{E_\mathrm{f}}$ containing those $(a,\tilde{x}_0)$ for which we needed to stop the experiment prematurely during \textbf{S3}.

\fakepar{S1}
The original \safeopt algorithm is obtained by choosing the next sample location
\begin{align}
\label{eqn:safeopt_next_point}
a_n = \argmax_{(a,\tilde{x}_0=x_0)\in G_n\cup M_n}\max_{i\in\mathcal{I}}w_n(a,x_0,i),
\end{align}
where $w_n(a,x_0,i) = u_n(a,x_0,i) - l_n(a,x_0,i)$.
This corresponds to \ref{item:S1}, where we solely explore in parameter space while fixing the initial condition to $x_0$.

\fakepar{S2}
Next, we jointly explore the parameter and the initial condition space, \ie we consider $\tilde{x}_0$.
However, we only seek to expand the safe set and not to find a new optimum.
Thus, we only consider the set of expanders $G_n$:
\begin{align}
\label{eqn:safeinit_next_point}
(a_n,\tilde{x}_0) &= \argmax_{(a,\tilde{x}_0)\in G_n}\max_{i\in\mathcal{I}_\mathrm{g}}w_n(a,\tilde{x}_0,i).
\end{align}

\fakepar{S3} 
We then globally sample parameters $a$ from $\mathcal{A}$ to find new safe areas, but still choose $\tilde{x}_0$ from $S_n$,
\begin{align}
\label{eqn:global_init_next_point}
(a_n,\tilde{x}_0) &= \argmax_{(a,\tilde{x}_0)\notin\mathcal{E_\mathrm{f}}\cup S_n: a\in \mathcal{A} \land \exists a': (a', \tilde{x}_0) \in S_n}\max_{i\in\mathcal{I}}w_n(a,\tilde{x}_0,i).
\end{align}
During \textbf{S3}, we exclude pairs $(a,\x0)$ that we already now are safe or for which we needed to interrupt the experiment in previous iterations.

\fakepar{Experiments}
The routine for carrying out an experiment is described in the \enquote{rollout} procedure in Alg.~\ref{alg:global_safeopt}.
During \ref{item:S3}, we need to monitor the state's evolution to intervene in case of potential constraint violation, \ie in case $x(t)$ hits the border of the safe set.
If $\exists t: [x(t)]_\mu\in\partial S_n$, we switch to a safe backup policy.
That is, we choose parameters $a'$ such that $(a', [x(t)]_\mu)\in S_n$. 
Further, we add $(a,\tilde{x}_0)$ of the experiment to $\mathcal{E_\mathrm{f}}$.
After successful experiments (\ie experiments during which we did not need to intervene), we 
update the safe set  
\begin{align}
    \label{eqn:update_safe_set_global} 
    S_n = S_{n-1} \cup (a_n, \tilde{x}_{0,n}).
\end{align}


\fakepar{Revisiting parameters}
Prematurely stopping an experiment does not imply that the experiment would have failed.
Thus, if the initial condition space of the safe set is increased, we check for all prematurely stopped experiments whether they would still be interrupted given the new safe set.
If this is not the case for some experiments, we delete those $(a,\x0)$ from $\mathcal{E_\mathrm{f}}$.
The acquisitions function~\eqref{eqn:global_init_next_point} may then suggest revisiting these points during subsequent steps.

\begin{algorithm}
\small 
\caption{Pseudocode of \ourmethod.}
\label{alg:global_safeopt}
\begin{algorithmic}[1]
\State \textbf{Input:} Safe seed $S_0$, initial condition $x_0$, tolerance $\epsilon$
\State $\mathcal{X}_\mathrm{fail} \gets \{\}$, $\mathcal{E_\mathrm{f}}\gets\{\}$
\For {$n=1,2,\ldots$} 
\If{$\max_{(a,\tilde{x}_0=x_0)\in G_{n-1}\cup M_{n-1},i\in\mathcal{I}}(w_n(a,x_0,i)) > \epsilon$}
\State Compute $a_n$ with~\eqref{eqn:safeopt_next_point}
\State safe, $x_\mathrm{fail}\gets\texttt{rollout}(a_n,x_0,S_{n-1})$
\ElsIf{$\max_{(a,\tilde{x}_0)\in G_{n-1},i\in\mathcal{I}_\mathrm{g}}(w_n(a,\tilde{x}_0,i))\! >\! \epsilon$}
\State Compute $(a_n, \tilde{x}_{0,n})$ with~\eqref{eqn:safeinit_next_point}
\State safe, $x_\mathrm{fail}\gets\texttt{rollout}(a_n,\tilde{x}_{0,n},S_{n-1})$
\Else
\State Compute $(a_n, \tilde{x}_{0,n})$ with~\eqref{eqn:global_init_next_point}
\State $\text{safe}, x_\mathrm{fail}\gets \texttt{rollout}(a_n,\tilde{x}_{0,n},S_{n-1})$
\EndIf
\If{safe}
\State Update safe set with~\eqref{eqn:safeset_update} or~\eqref{eqn:update_safe_set_global}
\State Update GPs
\Else 
\State $\mathcal{X}_\mathrm{fail}\gets\mathcal{X}_\mathrm{fail}\cup x_\mathrm{fail}$, $\mathcal{E_\mathrm{f}}\gets \mathcal{E_\mathrm{f}}\cup (a_n,\tilde{x}_{0,n})$
\EndIf
\For{$x\in\mathcal{X}_\mathrm{fail}$}
\If{$x\notin \partial{S}_n$}
\State $\mathcal{X}_\mathrm{fail}\gets\mathcal{X}_\mathrm{fail}\setminus x$, $\mathcal{E_\mathrm{f}}\gets\mathcal{E_\mathrm{f}}\setminus (a,\x0)$
\EndIf
\EndFor
\EndFor
\Return Best guess: $\hat{a}\gets \argmax_{(a, \tilde{x}_0=x_0)\in S_n}l_n(a, x_0,0)$
\Procedure{rollout}{$a$, $x_0$, $S$}
\State safe $\gets$ True, $x_\mathrm{fail}\gets$ None
    \While {$t\le T$}
    \State {$x(t)\gets x_0+\int_0^tz(x(\tau),\pi(x(\tau);a))\diff\tau$}
    \If {$(a, x_0)\notin S \text{ and } [x(t)]_\mu\in\partial S$}
    \State $a\gets a',$ such that $(a', [x(t)]_\mu) \in S$
    \State safe $\gets$ False, $x_\mathrm{fail}\gets x(t)$
    \EndIf
    \EndWhile
\Return safe, $x_\mathrm{fail}$
\EndProcedure
\end{algorithmic}
\end{algorithm}

\subsection{Theoretical Guarantees}
In this section, we show that \ourmethod enjoys the same safety guarantees as \safeopt, while we state conditions under which we can recover the global optimum\footnote{Detailed proofs can be found in the appendix.}.

\begin{theo}
\label{thm:glob_safety_init_cond}
Assume that $h(a,\tilde{x}_0,i)$ with $i\in\mathcal{I}_g$ has a norm bounded by $B$ in an RKHS and that the measurement noise is $\sigma$-sub Gaussian. 
Further, assume that $S_0\neq\varnothing$.
Choose $\beta_n^{\sfrac{1}{2}}=B+4\sigma\sqrt{\gamma_{(n-1)\abs{\mathcal{I}}}+1+\ln(\sfrac{1}{\delta})}$, where $\gamma_n$ describes the \emph{information capacity} associated with a kernel (see~\cite{berkenkamp2016bayesian} for more details). 
Then, given the proposed algorithm, with probability at least $1-\delta$, where $\delta\in(0,1)$, we have $\forall n\ge 1, \forall t, \forall i\in\mathcal{I}_g:\bar{g}_i(x_n(t))\ge 0.$
\end{theo}%
\begin{proof}
During \ref{item:S1} and \ref{item:S2}, we always have $(a_n, \tilde{x}_{0,n})\in S_n$, thus, safety follows directly from Theorem~1 in~\cite{berkenkamp2016bayesian}.
For \ref{item:S3}, we switch to $(a',[x(t)]_\mu)\in S_n$ if $\exists t:[ x(t)]_\mu\in\partial{S}_n$.
Then, safety also follows from Theorem~1 in~\cite{berkenkamp2016bayesian}.
If $\nexists t: [x(t)]_\mu\in\partial{S}_n$, safety follows from the safe set definition.
\end{proof}
We further provide conditions under which we can find the global optimum.
For a general system~\eqref{eqn:sysdyn}, the globally optimal parameters may drive the system state outside of the safe area during the transient.
In this case, we would prematurely stop the experiment, switch to a safe backup policy, and mark the experiment as failed.
To guarantee optimality, we thus need to assume that the trajectory for the globally optimal parameters $a$ lies within the largest safe area in state space that can be learned without risking a failure.
To formally define this, we first adapt the one-step reachability operator from~\cite{sui2015safe} to also include the initial condition space, $\Rcon(S) \coloneqq S \,\cup\, \{a \in \mathcal{A}, \tilde{x}_0\in\mathcal{X}_\mu\mid \exists (a', \tilde{x}_0')\in S\text{ such that } g_i(a', \tilde{x}_0)-\epsilon-L_\mathrm{a}\norm{a'-a}-L_\mathrm{x}(\norm{\tilde{x}_0'-\tilde{x}_0}+\mu)\ge 0\,\forall i\in\mathcal{I}_\mathrm{g}\}$. 
This denotes the set of points that can be established as safe using safe explorations (\ie during \ref{item:S1} and \ref{item:S2}), when $g_i$ can be measured with $\epsilon$-precision.
We now extend this reachability operator to account for the fact that we can also expand globally.
The global reachability operator $\Rglob$, thus, in addition, comprises those $(a, \tilde{x}_0)$ pairs for which a safe backup policy with parameters $a'$ exists and whose trajectory does not hit the boundary of the safe set.
This global reachability operator is defined as $\Rglob(S) \coloneqq \Rcon(S) \cup \{a\in\mathcal{A},\tilde{x}_0\in\mathcal{X}_\mu\mid\exists a'\in\mathcal{A}\text{ such that }(a',\tilde{x}_0)\in S\text{ and } [\tilde{x}_0+\int_0^{t}z(x(\tau),\pi(x(\tau);a)\diff \tau]_\mu\notin \partial{\Rcon(S)}\,\forall t\}$.
By defining $R^n_\epsilon(S)$ as the repeated composition of $\Rglob(S)$ with itself, we can further obtain its closure as $\bar{R}_\epsilon(S_0)\coloneqq\lim_{n\to\infty}\R_\epsilon^n(S_0)$.

\begin{theo}
\label{thm:glob_optimality}
Assuming the same as in Theorem~\ref{thm:glob_safety_init_cond}, \ourmethod converges to the optimum within $\bar{R}_\epsilon(S_0)$ with $\epsilon$-precision with probability at least $1-\delta$.
\end{theo}
\begin{proof}
During \textbf{S1} and \textbf{S2}, the convergence of the safe set to the maximum safely reachable region connected to it  with $\epsilon$-precision follows from~\cite{berkenkamp2016bayesian}.
If there are further, disconnected safe points $(a,\x0)\in\bar{R}_\epsilon(S_n)\setminus S_n$, those will eventually be explored during \textbf{S3} since  $\mathcal{X}_\mu$ and $\mathcal{A}$ are finite.
The alternation of these steps guarantees convergence to $\bar{R}_\epsilon(S_0)$. Then, finding the optimum within $\bar{R}_\epsilon(S_0)$ with $\epsilon$-precision follows from~\cite{berkenkamp2016bayesian}.
\end{proof}
\begin{cor}
\label{cor:glob_optimality}
Let $a^*=\argmax_{a\in\mathcal{A}}f(a,x_0)$ subject to $g(a,x_0)\ge 0$.
If we have $[x_0+\int_0^t z(x(\tau), \pi(x(\tau);a^*)\diff \tau]_\mu\notin \partial \bar{R}_\epsilon(S_0)$ for all $t$, then $(a^*,x_0)\in\bar{R}_\epsilon(S_0)$ and \ourmethod is guaranteed to find an $\epsilon$-optimal solution.
\end{cor}

\begin{figure*}
     \centering
     \subfloat[GP posterior mean of $f(x)$. \label{sfig:posterior}]{
         \includegraphics[width=0.32\textwidth]{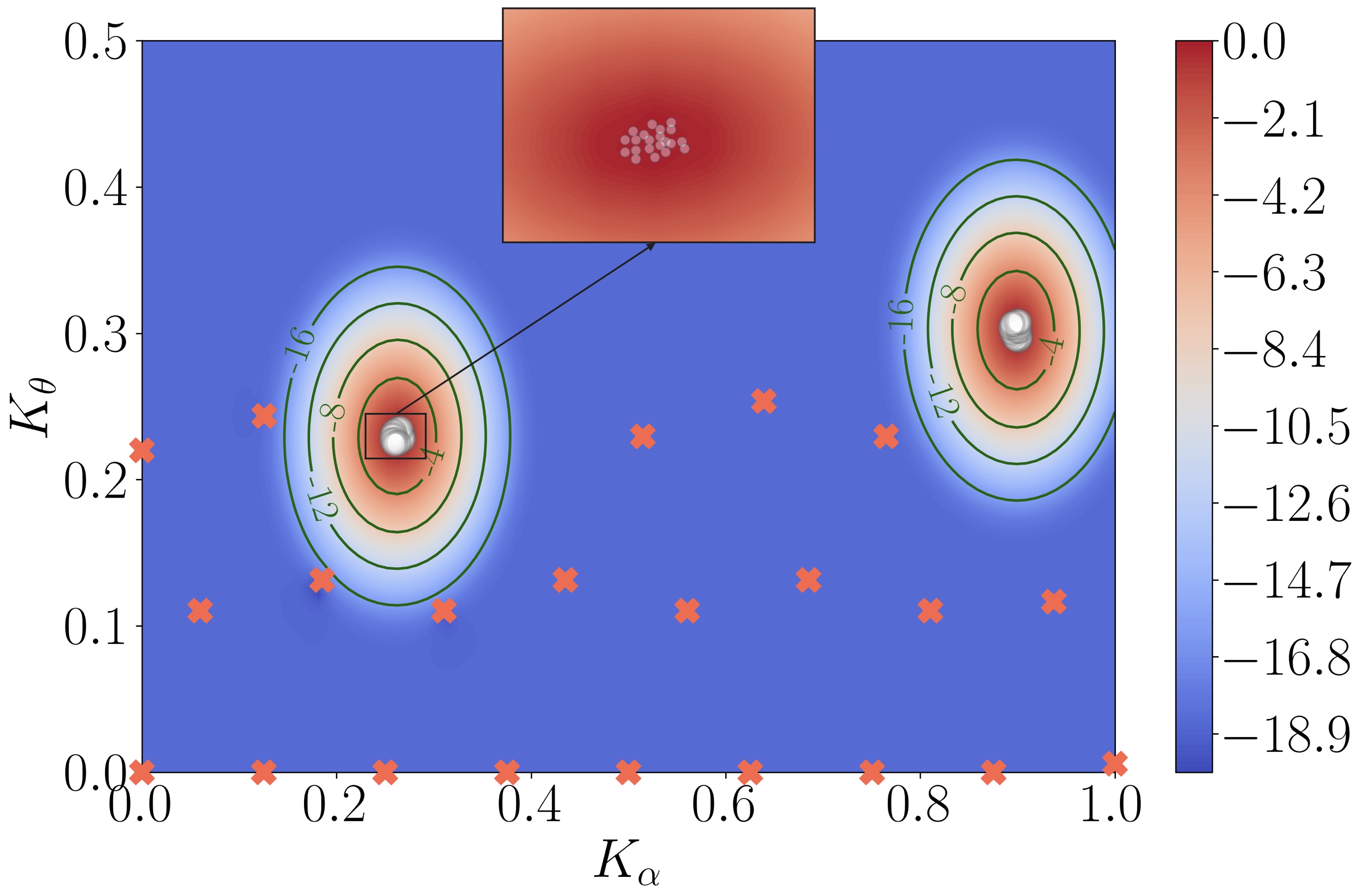}
         }
     \subfloat[Reward evolution over iterations. \label{sfig:regret}]{
         \includegraphics[width=0.32\textwidth]{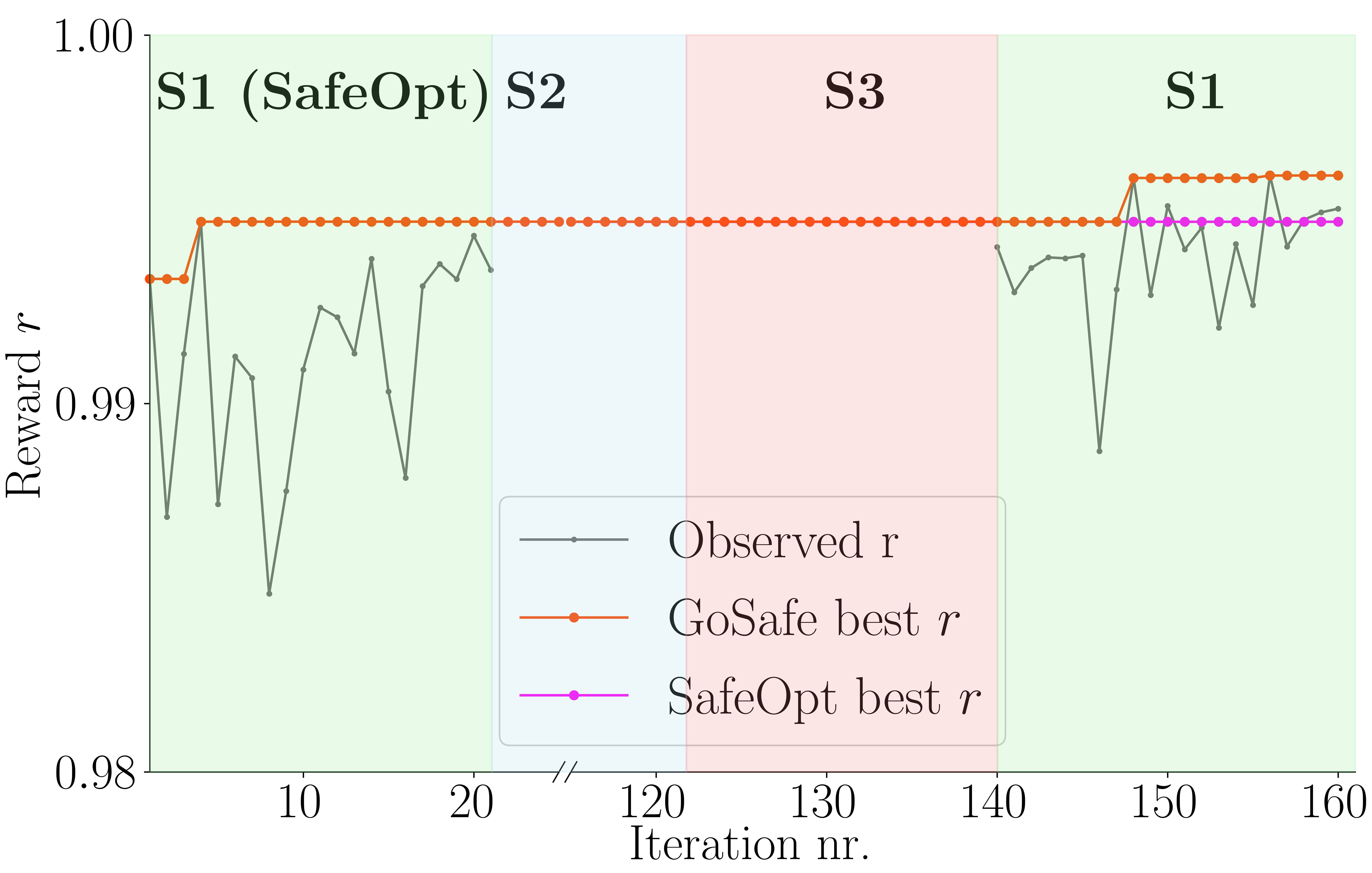}
         }
     \subfloat[Best guess comparison. \label{sfig:best_guess}]{
         \includegraphics[width=0.30\textwidth]{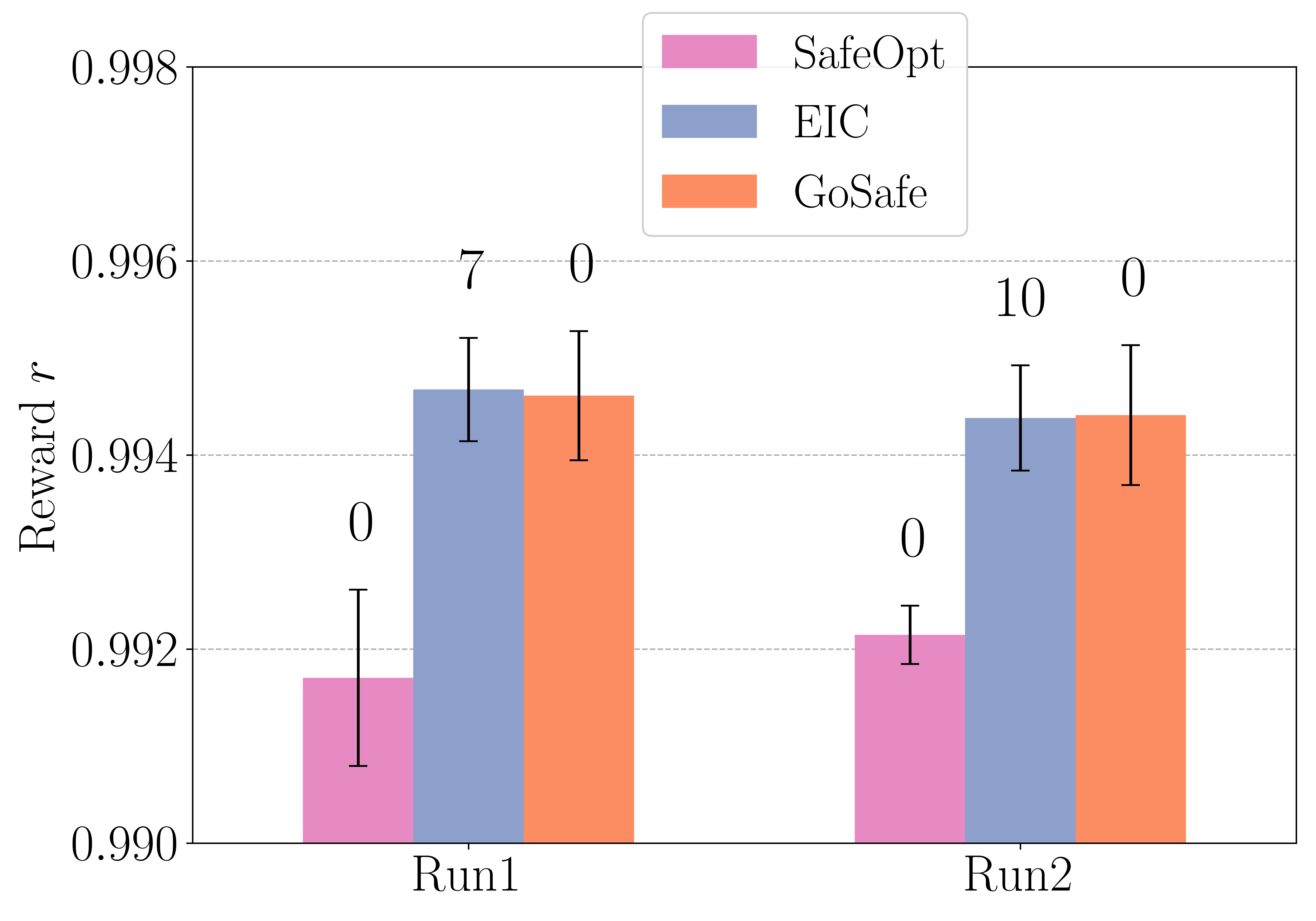}
         }
        \caption{Evaluation of \ourmethod. (a) \capt{The algorithm identifies distinct safe areas. The contours show the posterior mean of the GP on a logarithmic scale, red crosses mark stopped experiments, and white circles successful ones. (b) Evolution of the reward over iterations during the different stages, showing the reward of the current evaluation and the best observation so far of \safeopt and \ourmethod. (c) Mean and standard deviation of three evaluations of the best controller found with \ourmethod, \safeopt, and EIC in two independent runs, as well as the number of failures written on top of the bar plots.}}
        \label{fig:eval}
\end{figure*}

\subsection{Practical Implementation}
\label{sec:practical_implementation}

The guarantees stated above come at a cost: the safe set needs to be expanded as much as possible, resulting in many experiments.
Further, we assume that we can continuously monitor the state.
In this section, we discuss how this can be dealt with in practical applications.

\fakepar{Reducing the search space}
The safe set $S_n$ contains pairs of policy parameters and initial conditions.
While we fix the initial condition to $x_0$ when searching for the optimum, we consider both $a$ and $\tilde{x}_0$ as free parameters when expanding in the joint initial condition and parameter space.
Instead, one can also choose fixed policy parameters when expanding in initial condition space.
The chosen parameters should then be likely to yield a large safe area, which could, for instance, be the parameters with the best performance obtained so far.

\fakepar{Reducing the number of experiments}
Expanding the initial condition space of the safe set may be terminated earlier to reduce the number of experiments.
Termination can be triggered by a stopping criterion that determines when the safe set is \enquote{big enough} or by executing \ref{item:S2} and \ref{item:S3} for a fixed amount of iterations.
Then, one can fix the initial condition to $x_0$ and directly search for the global optimum using, \eg 
\begin{align}
\label{eqn:global_max_next_point}
a_n = \argmax_{a\in \mathcal{A}}\max_{i\in\mathcal{I}}w_n(a,x_0,i).
\end{align}

\fakepar{Sampling interval}
In practice, we only receive sampled measurements of $x(t)$.
Thus, we introduce a parameter $\eta$ and interrupt experiments if $\inf_{y\in \partial{S}_n}\norm{x(t)-y} < \eta$.
The threshold must be chosen in accordance with the dynamics of the system and the sampling rate.

\section{Evaluation}
\label{sec:eval}

We evaluate \ourmethod on a real Furuta pendulum~\cite{furuta1992swing}.
The Furuta pendulum is a fast and underactuated, unstable system and, therefore, a challenging platform for \ourmethod.

\fakepar{Setting}
We consider stabilization of the pendulum using static state feedback $u(t)=K x(t)$, where the entries of $K$ that multiply the pole angle ($K_\theta$) and the angle of the rotatory arm ($K_\alpha$) are learned in the range $[0,1]$, while the rest are fixed.
We model $f$ and $g$ with independent GPs with Mat\'ern 3/2 kernels with length scale $0.01$ and variance $5$ for the parameter space, and $0.1$ and $0.2$ for the initial condition space.
We choose $\beta_n\equiv 3$ and constraints are given by $\abs{\alpha}<\sfrac{\pi}{2}$ and $\abs{\theta}<\sfrac{\pi}{2}$.
Accordingly, $g_1$ is defined as the minimum distance to either of these constraints during an experiment.
The reward function in case of a successful experiment, $R = \frac{1}{T}\sum_{t=0}^{T-1}(1-\frac{0.8\abs{\theta(t)} + 0.2\abs{\alpha(t)}}{\pi})$, is strictly positive.
If the backup controller needs to kick in, we return a reward and constraint value of 0 and include both in the GP.
During experiments, we receive measurements at a frequency of \SI{50}{\hertz}, \ie considerably slower as the \SI{1}{\kilo\hertz} typically used in robotics.
We switch to the backup controller whenever $l_n(a,x(t),1)<1$.
Our implementation builds upon the original \safeopt code\footnote{\url{https://github.com/befelix/SafeOpt}}.

\looseness=-1
\fakepar{Results}
We initialize the algorithm in the safe area on the right in \figref{sfig:posterior} and then execute \ref{item:S1} for 20 iterations.
As highlighted in \figref{sfig:regret}, \ref{item:S1} corresponds to standard \safeopt. 
Thus, the algorithm stays inside the initial safe region.
While the controller found at the end of \ref{item:S1} is the best \safeopt can find, \ourmethod can now progress to \ref{item:S2} and \ref{item:S3} to discover new safe areas, which naturally comes at the cost of more iterations.
Thus, we next execute \ref{item:S2} for 100 iterations.
While the state space is of higher dimension, the acquisition function of \ref{item:S2} is substantially easier than that of \ref{item:S1}.
Therefore, computing the next sample location is even faster in \ref{item:S2} as compared to \ref{item:S1}.
Then, \ourmethod enters \ref{item:S3}, where we globally search for new safe areas.
We adopt the modifications proposed in \secref{sec:practical_implementation}.
In particular, we use~\eqref{eqn:global_max_next_point} to search for new safe areas for the nominal initial conditions only.
Most of these experiments are unsafe and terminated early since the system gets too close to $\partial S_n$. 
In such cases, the backup policy was triggered to keep the pole in balance.
Eventually, a second safe area (left in \figref{sfig:posterior}) is detected and subsequently expanded by revisiting \ref{item:S1} (\cf \figref{sfig:regret}).
That is, we execute \safeopt in the newly found safe region. However, the original \safeopt algorithm would not have been able to find a second safe area.
In this second area, the algorithm finds a controller parameterization with a slightly better reward than the optimum in the initial safe area.
Thus, \ourmethod reveals a better optimum than could have been found using \safeopt, as can also be confirmed over multiple runs, as shown in \figref{sfig:best_guess}, at the cost of the additional iterations in \ref{item:S2} and \ref{item:S3} (\cf \figref{sfig:regret}).
During multiple runs of the whole algorithm, the pole of the Furuta pendulum never dropped.
The results demonstrate that \ourmethod can find additional safe areas and, thus, find optima that the original \safeopt algorithm would miss, while still providing safety guarantees\footnote{Video available at \url{https://youtu.be/YgTEFE_ZOkc}}.

In addition to the above, we also compared \ourmethod against expected improvement with constraints (EIC) \cite{Gelbart2014}. This is a popular BOC algorithm that suggests data points where the objective is expected to improve the best observed reward so far while discouraging regions of potential constraint violation.
We adopt the hyperparameters from \ourmethod, however, choose a length scale of 0.1 for the parameter space GP to enable faster exploration.
As can be inferred from \figref{sfig:best_guess}, EIC finds a similar optimum as \ourmethod.
While EIC, due to the increased length scale, reaches that optimum significantly faster than \ourmethod, needing only 28 and 18 iterations in two independent runs to find the best guesses shown in \figref{sfig:best_guess}, it incurs in a significant amount of failures (10 and 7, respectively) during exploration.
This would cause severe hardware damage on a more fragile robotic system.

The comparison with EIC and \safeopt reveals that \ourmethod retains the best of both worlds: it finds the global optimum \emph{and} guarantees safety during exploration.

\section{Conclusions}
\label{sec:concl}

We proposed \ourmethod, an extension of the well-known \safeopt algorithm.
Both algorithms aim at optimizing a policy while guaranteeing safety during exploration.
By considering safety not only in parameter but also in state space, \ourmethod can, in contrast to \safeopt, explore globally, switching to a safe backup policy in case of potential constraint violation.
Experiments on a Furuta pendulum demonstrate that \ourmethod finds the global optimum also in case it is outside an initially given safe area, while \safeopt remains stuck at the local optimum inside the initial region.
\bibliographystyle{IEEEtran}
\bibliography{IEEEabrv,ref}

\begin{thebibliography}{10}
\providecommand{\url}[1]{#1}
\csname url@rmstyle\endcsname
\providecommand{\newblock}{\relax}
\providecommand{\bibinfo}[2]{#2}
\providecommand\BIBentrySTDinterwordspacing{\spaceskip=0pt\relax}
\providecommand\BIBentryALTinterwordstretchfactor{4}
\providecommand\BIBentryALTinterwordspacing{\spaceskip=\fontdimen2\font plus
\BIBentryALTinterwordstretchfactor\fontdimen3\font minus
  \fontdimen4\font\relax}
\providecommand\BIBforeignlanguage[2]{{%
\expandafter\ifx\csname l@#1\endcsname\relax
\typeout{** WARNING: IEEEtran.bst: No hyphenation pattern has been}%
\typeout{** loaded for the language `#1'. Using the pattern for}%
\typeout{** the default language instead.}%
\else
\language=\csname l@#1\endcsname
\fi
#2}}

\bibitem{berkenkamp2016safe}
F.~Berkenkamp, A.~P. Schoellig, and A.~Krause, ``Safe controller optimization
  for quadrotors with {G}aussian processes,'' in \emph{IEEE International
  Conference on Robotics and Automation}, 2016, pp. 491--496.

\bibitem{gryazina2006stability}
E.~N. Gryazina and B.~T. Polyak, ``Stability regions in the parameter space:
  D-decomposition revisited,'' \emph{Automatica}, vol.~42, no.~1, pp. 13--26,
  2006.

\bibitem{calandra2016bayesian}
R.~Calandra, A.~Seyfarth, J.~Peters, and M.~P. Deisenroth, ``Bayesian
  optimization for learning gaits under uncertainty,'' \emph{Annals of
  Mathematics and Artificial Intelligence}, vol.~76, no. 1-2, pp. 5--23, 2016.

\bibitem{mockus1978application}
J.~Mockus, V.~Tiesis, and A.~Zilinskas, ``The application of {B}ayesian methods
  for seeking the extremum,'' \emph{Towards Global Optimization}, vol.~2, no.
  117-129, p.~2, 1978.

\bibitem{antonova2017deep}
R.~Antonova, A.~Rai, and C.~G. Atkeson, ``Deep kernels for optimizing
  locomotion controllers,'' \emph{arXiv preprint arXiv:1707.09062}, 2017.

\bibitem{marco2016automatic}
A.~Marco, P.~Hennig, J.~Bohg, S.~Schaal, and S.~Trimpe, ``Automatic {LQR}
  tuning based on {G}aussian process global optimization,'' in \emph{IEEE
  International Conference on Robotics and Automation}, 2016, pp. 270--277.

\bibitem{turchetta2019robust}
M.~Turchetta, A.~Krause, and S.~Trimpe, ``Robust model-free reinforcement
  learning with multi-objective {B}ayesian optimization,'' \emph{arXiv preprint
  arXiv:1910.13399}, 2019.

\bibitem{sui2015safe}
Y.~Sui, A.~Gotovos, J.~Burdick, and A.~Krause, ``Safe exploration for
  optimization with {G}aussian processes,'' in \emph{International Conference
  on Machine Learning}, 2015, pp. 997--1005.

\bibitem{berkenkamp2016bayesian}
F.~Berkenkamp, A.~Krause, and A.~P. Schoellig, ``Bayesian optimization with
  safety constraints: safe and automatic parameter tuning in robotics,''
  \emph{arXiv preprint arXiv:1602.04450}, 2016.

\bibitem{schreiter2015safe}
J.~Schreiter, D.~Nguyen-Tuong, M.~Eberts, B.~Bischoff, H.~Markert, and
  M.~Toussaint, ``Safe exploration for active learning with {G}aussian
  processes,'' in \emph{Joint European Conference on Machine Learning and
  Knowledge Discovery in Databases}, 2015, pp. 133--149.

\bibitem{schillinger2018safe}
M.~Schillinger, B.~Ortelt, B.~Hartmann, J.~Schreiter, M.~Meister,
  D.~Nguyen-Tuong, and O.~Nelles, ``Safe active learning of a high pressure
  fuel supply system,'' in \emph{EUROSIM Congress on Modelling and Simulation},
  no. 142, 2018, pp. 286--292.

\bibitem{berkenkamp2017safe}
F.~Berkenkamp, M.~Turchetta, A.~Schoellig, and A.~Krause, ``Safe model-based
  reinforcement learning with stability guarantees,'' in \emph{Advances in
  Neural Information Processing Systems}, 2017, pp. 908--918.

\bibitem{berkenkamp2016roa}
F.~Berkenkamp, R.~Moriconi, A.~P. Schoellig, and A.~Krause, ``Safe learning of
  regions of attraction for uncertain, nonlinear systems with {G}aussian
  processes,'' in \emph{IEEE Conference on Decision and Control}, 2016, pp.
  4661--4666.

\bibitem{akametalu2014reachability}
A.~K. Akametalu, J.~F. Fisac, J.~H. Gillula, S.~Kaynama, M.~N. Zeilinger, and
  C.~J. Tomlin, ``Reachability-based safe learning with {G}aussian processes,''
  in \emph{IEEE Conference on Decision and Control}, 2014, pp. 1424--1431.

\bibitem{turchetta2016safe}
M.~Turchetta, F.~Berkenkamp, and A.~Krause, ``Safe exploration in finite
  {M}arkov decision processes with {G}aussian processes,'' in \emph{Advances in
  Neural Information Processing Systems}, 2016, pp. 4312--4320.

\bibitem{turchetta2019safe}
M.~Turchetta, F.~Berkenkamp, and A.~Krause, ``Safe exploration for interactive
  machine learning,'' in \emph{Advances in Neural Information Processing
  Systems}, 2019, pp. 2891--2901.

\bibitem{heim2019learnable}
S.~Heim, A.~von Rohr, S.~Trimpe, and A.~Badri-Spr{\"o}witz, ``A learnable
  safety measure,'' in \emph{Conference on Robot Learning}, 2019.

\bibitem{achiam2017constrained}
J.~Achiam, D.~Held, A.~Tamar, and P.~Abbeel, ``Constrained policy
  optimization,'' in \emph{International Conference on Machine Learning}, 2017,
  pp. 22--31.

\bibitem{chow2018lyapunov}
Y.~Chow, O.~Nachum, E.~Duenez-Guzman, and M.~Ghavamzadeh, ``A {L}yapunov-based
  approach to safe reinforcement learning,'' in \emph{Advances in Neural
  Information Processing Systems}, 2018, pp. 8092--8101.

\bibitem{garcia2015comprehensive}
J.~Garc{\i}a and F.~Fern{\'a}ndez, ``A comprehensive survey on safe
  reinforcement learning,'' \emph{Journal of Machine Learning Research},
  vol.~16, no.~1, pp. 1437--1480, 2015.

\bibitem{hernandez2016general}
J.~M. Hern\'{a}ndez-Lobato, M.~A. Gelbart, R.~P. Adams, M.~W. Hoffman, and
  Z.~Ghahramani, ``A general framework for constrained {B}ayesian optimization
  using information-based search,'' \emph{The Journal of Machine Learning
  Research}, vol.~17, no.~1, pp. 5549--5601, 2016.

\bibitem{Gelbart2014}
M.~A. Gelbart, J.~Snoek, and R.~P. Adams, ``Bayesian optimization with unknown
  constraints,'' in \emph{Conference on Uncertainty in Artificial
  Intelligence}, 2014, pp. 250--259.

\bibitem{gardner2014bayesian}
J.~R. Gardner, M.~J. Kusner, Z.~E. Xu, K.~Q. Weinberger, and J.~P. Cunningham,
  ``Bayesian optimization with inequality constraints,'' in \emph{International
  Conference on Machine Learning}, 2014, pp. 937--945.

\bibitem{gramacy2011opti}
R.~B. Gramacy and H.~Lee, ``Optimization under unknown constraints,''
  \emph{Bayesian Statistics 9}, 2011.

\bibitem{schonlau1998global}
M.~Schonlau, W.~J. Welch, and D.~R. Jones, ``Global versus local search in
  constrained optimization of computer models,'' \emph{Lecture Notes-Monograph
  Series}, pp. 11--25, 1998.

\bibitem{picheny2014stepwise}
V.~Picheny, ``A stepwise uncertainty reduction approach to constrained global
  optimization,'' in \emph{International Conference on Artificial Intelligence
  and Statistics}, 2014, pp. 787--795.

\bibitem{marco2021robot}
A.~Marco, D.~Baumann, M.~Khadiv, P.~Hennig, L.~Righetti, and S.~Trimpe, ``Robot
  learning with crash constraints,'' \emph{IEEE Robotics and Automation
  Letters}, 2021.

\bibitem{marco2020excursion}
A.~Marco, A.~von Rohr, D.~Baumann, J.~M. Hern{\'a}ndez-Lobato, and S.~Trimpe,
  ``Excursion search for constrained {B}ayesian optimization under a limited
  budget of failures,'' \emph{arXiv preprint arXiv:2005.07443}, 2020.

\bibitem{schlkopf2018learning}
B.~Sch\"olkopf, A.~J. Smola, and F.~Bach, \emph{Learning with kernels: support
  vector machines, regularization, optimization, and beyond}.\hskip 1em plus
  0.5em minus 0.4em\relax MIT Press, 2018.

\bibitem{srinivas2012information}
N.~Srinivas, A.~Krause, S.~M. Kakade, and M.~W. Seeger, ``Information-theoretic
  regret bounds for {G}aussian process optimization in the bandit setting,''
  \emph{{IEEE} Trans. Inform. Theory}, vol.~58, no.~5, pp. 3250--3265, 2012.

\bibitem{Chowdhury2017OnKM}
S.~R. Chowdhury and A.~Gopalan, ``On kernelized multi-armed bandits,'' in
  \emph{International Conference on Machine Learning}, 2017, pp. 844--853.

\bibitem{williams2006gaussian}
C.~K. Williams and C.~E. Rasmussen, \emph{Gaussian Processes for Machine
  Learning}.\hskip 1em plus 0.5em minus 0.4em\relax MIT press Cambridge, MA,
  2006, vol.~2, no.~3.

\bibitem{furuta1992swing}
K.~Furuta, M.~Yamakita, and S.~Kobayashi, ``Swing-up control of inverted
  pendulum using pseudo-state feedback,'' \emph{Proceedings of the Institution
  of Mechanical Engineers, Part I: Journal of Systems and Control Engineering},
  vol. 206, no.~4, pp. 263--269, 1992.

\end{thebibliography}

\newpage
\appendix

In this appendix, we provide extended proofs for Theorems~\ref{thm:glob_safety_init_cond} and~\ref{thm:glob_optimality}, and Corollary~\ref{cor:glob_optimality}.
For ease of notation, we denote by $\xi_{(t_0,\x0,a)}$ the trajectory of $x(t)$ starting at $x(t_0) = \x0$ with parameters $a$, \ie $\xi_{(t_0,\x0,a)}\coloneqq\{x_0 + \int_{t_0}^t z(x(\tau),\pi(x(\tau);a)\diff\tau\, \forall t\geq t_0\}$.
All notation is summarized in \tabref{tab:notation}.

\begin{table}
    \centering
    \caption{Summary of main variables and operators.}
    \label{tab:notation}
    \begin{tabular}{ll}
        \toprule
        Symbol & Explanation\\
        \midrule
        $t$ & time\\
        $z$ & System dynamics \\
        $x$ & System state \\
        $a$ & Policy parameters\\
        $f$ & Reward function \\
        $g_i$ & Constraint function $i$\\
        $\bar{g}_i$ & Immediate constraint function $i$\\
        $h_i$ & Surrogate selector function for index $i$\\
        $\mathcal{I}$ & Set of all indices\\
        $\mathcal{I}_\mathrm{g}$ & Set of all indices pertaining to constraints (\ie $i>0$)\\
        $\mathcal{A}$ & Parameter space\\
        $\mathcal{X}$ & Continuous state space\\
        $\mu$ & Discretization parameter\\
        $\mathcal{X}_\mu$ & Discretized state space\\
        $[x]_\mu$ & $x'\in\mathcal{X}_\mu$ with smallest $L_1$-distance to $x\in\mathcal{X}$\\
        $u_n(a,\x0,i)$ & upper bound of the confidence interval at iteration $n$ \\
        $l_n(a, \x0, i)$ & lower bound of the confidence interval at iteration $n$\\
        $S_n$ & Safe set at iteration $n$\\
        $\partial S_n$ & Border of safe set at iteration $n$\\
        $\Rcon$ & Reachability operator for \textbf{S1} and \textbf{S2}\\
        $\Rglob$ & Reachability operator for \textbf{S3}\\
        $\bar{R}$ & Closure of $R$\\
        $\traj{t_0}{\x0}{a}$ & Trajectory of $x$ starting in $\x0$ with parameters $a$\\
        $L_\mathrm{a},L_\mathrm{x}$ & Lipschitz constants for parameter and state space\\
        \bottomrule
    \end{tabular}
\end{table}

The proofs largely depend on the fact that we have accurate uncertainty estimates of the reward function $f$ and the constraint function $g$.
Both are contained in the surrogate selector function $h$ (\cf~\eqref{eqn:gp_surrogate}).
Under the assumption that $h$ has bounded RKHS norm, we can state:
\begin{lem}
\label{lem:uncertainty_bound}
Assume that the RKHS norm of $h(a,\x0, i)$ is bounded by $B$ and that measurements are corrupted by $\sigma$-sub-Gaussian noise.
If $\beta_n^{\sfrac{1}{2}}=B=4\sigma\sqrt{\gamma_{n-1}\abs{I}+1+\ln(1/\delta)}$, where $\gamma$ is the information capacity (see~\cite{berkenkamp2016bayesian} for a discussion), then the following holds for all parameters $a\in\mathcal{A}$, initial conditions $\x0\in\mathcal{X}_\mu$, function indices $i\in\mathcal{I}$, and iterations $n\ge 1$ jointly with probability at least $1-\delta$:
\begin{align}
    \label{eqn:uncertainty_bound}
    \abs{h(a, \x0, i)-\mu_{n-1}(a,\x0,i)}\le \beta_n^{\sfrac{1}{2}}\sigma_{n-1}(a,\x0,i).
\end{align}
\end{lem}
\begin{proof}
Directly follows from~\cite{Chowdhury2017OnKM}, with the difference that we obtain $\abs{\mathcal{I}}$ measurements per iteration, causing a faster growth of the information capacity $\gamma$ (\cf~\cite[Lem.~1]{berkenkamp2016bayesian}). 
\end{proof}
In the following proofs, we implicitly assume that assumptions of Lemma~\ref{lem:uncertainty_bound} hold when needed and that $\beta_n$ is defined as therein.
\begin{cor}
\label{cor:uncertainty_bound}
For $\beta_n$ as in Lemma~\ref{lem:uncertainty_bound}, we have with probability at least $1-\delta$ and for all $n\ge1, i\in\mathcal{I}, a\in\mathcal{A}, \x0\in\mathcal{X}_\mu$ that $h(a,x,i)\in C_n(a,\x0,i)$.
\end{cor}
\begin{proof}
The proof straightforwardly follows from~\cite[Cor.~1]{berkenkamp2016bayesian}.
By Lemma~\ref{lem:uncertainty}, we have that the true functions are contained in $Q_n(a,\x0,i)$ with probability at least $1-\delta$.
Thus, the true function is also contained in the intersection of these sets with the same probability.
\end{proof}

\subsection{Extended Proof of Theorem~\ref{thm:glob_safety_init_cond}}

To prove Theorem~\ref{thm:glob_safety_init_cond}, we first restate our assumptions on the Lipschitz-continuity. 
\begin{assume}
\label{ass:Lipschitz}
The dynamics $z$ of the system, as well as the constraint function $g$ and the reward function $f$ are Lipschitz-continuous with Lipschitz parameters $L_\mathrm{x}$ and $L_\mathrm{a}$.
\end{assume}
We further make a few assumptions on pairs $(a,\x0)$ that belong to the safe set $S_n$.
Intuitively, we assume that all those points are safe with high probability and that for all pairs we have policy parameters that ensure that also all continuous states $x$ that are $\mu$-close to $[\x0]_\mu$ are safe.
\begin{assume}
\label{ass:safe_set}
Let $S_n\neq\varnothing$. 
The following properties hold for all $i\in\mathcal{I}_\mathrm{g}$ with probability at least $1-\delta$:
\begin{renum}
    \item $\forall (a,\x0) \in S_n: g_i(a,\x0)\ge0$;
    \item $\forall (a,\x0)\in S_n: \exists a'\in\mathcal{A}:(a',\x0)\in S_n \land g_i(a',\x0)\ge L_\mathrm{x}\mu$ ($a$ and $a'$ may be identical);
    \item $\forall \x0\in\partial S_n, a\in\mathcal{A}:(a,\x0)\in S_n$: $g_i(a,\x0)\ge L_\mathrm{x}\mu$.
\end{renum}
\end{assume}
For the initial $S_0$, this assumption is satisfied since we have $\forall (a,\x0)\in S_0: g_i(a,\x0)\ge L_\mathrm{x}\mu$ for all $i\in\mathcal{I}_\mathrm{g}$.
We first show that under Assumption~\ref{ass:safe_set}, we can guarantee safety during the global search in \textbf{S3}.
We then show that the update rules guarantee that Assumption~\ref{ass:safe_set} is satisfied for all $n$ given $S_0$.

During the global search, we may evaluate both safe and unsafe policy parameters.
To guarantee safety, we thus need to show that \emph{(i)} before leaving the safe set during an experiment, we trigger a safe backup policy, \emph{(ii)} if we do not trigger a backup policy, the evaluated pair $(a, \x0)$ is safe with high probability, \emph{(iii)} triggering a backup policy guarantees constraint satisfaction with high probability.
We start by proving that before leaving the safe set, we necessarily visit a state that belongs to the border $\partial S_n$ of the safe set.
\begin{lem} 
\label{lem:safety_switching}
Consider $\x0\in\mathcal{X}_\mu$ such that $\exists a'\in\mathcal{A}:(a',\x0)\in S_n$.
Let $x(t)$ denote the state of the system at time $t$ for the trajectory $\traj{t_0}{\x0}{a}$ for all $t\ge t_0\ge 0$ and with $a\in\mathcal{A}$.
Define $T_\mathrm{unsafe}\coloneqq \{t\in[t_0,\infty):\nexists a\in\mathcal{A} \text{ such that }(a,[x(t)]_\mu)\in S_n$\}.
Assume $T_\mathrm{unsafe}\neq \varnothing$ and let $t_2\coloneqq \min(T_\mathrm{unsafe})$.
Then $\exists t_1 \in [t_0, t_2)\text{ such that } [x(t_1)]_\mu\in\partial S_n\text{ and } \forall t'\in[t_0,t_1]\;\exists a'\in\mathcal{A} \text{ such that } (a',[x(t')]_\mu)\in S_n$.
\end{lem}
\begin{proof}
First, consider the case $x(t_0)=\x0\in\partial S_n$.
Then, we have $t'=t_0$ and $\exists a'\in\mathcal{A}:(a',\x0)\in S_n$ by the choice of $\x0$.
Now, assume $x(t_0)\notin\partial S_n$.
Then, we have, by definition of $\partial S_n$, $\norm{x(t_0)-x(t_2)}>2\mu.$
By definition of $t_2$, we further have that
\begin{align}
    \label{eqn:def_t2}
    \forall t' <t_2: \exists a'\text{ such that } (a',[x(t')]_\mu)\in S_n.
\end{align}
As $z$ is Lipschitz-continuous, there exists a unique solution to~\eqref{eqn:sysdyn} and continuity of $x(t)$ follows from the Picard-Lindel\"{o}f theorem.
Thus, $\exists t_1<t_2$ such that $\norm{x(t_2)-x(t_1)}<\mu$.
The discretization further implies that $\norm{x(t_1)-[x(t_1)]_\mu}\le \mu$.
It follows that
\begin{align}
    \label{eqn:dist_unsafe}
    \begin{split}
    &\norm{x(t_2)-[x(t_1)]_\mu}\\ 
    \le &\norm{x(t_2)-x(t_1)}+\norm{x(t_1)-[x(t_1)]_\mu}\\
    < &2\mu.
    \end{split}
\end{align}
Since $t_1<t_2$,~\eqref{eqn:def_t2} implies that $\exists a^{\mathrm{t_1}}\in\mathcal{A}$ such that $(a^{\mathrm{t_1}},[x(t_1)]_\mu)\in S_n.$
Taking this and~\eqref{eqn:dist_unsafe} yields $[x(t_1)]_\mu\in\partial S_n$.
This together with~\eqref{eqn:def_t2} completes the proof.
\end{proof}
Thus, before leaving the safe set, we necessarily visit a state that belongs to $\partial S_n$.

Next, we show that for experiments for whose entire duration we stay inside the safe set, the immediate constraint is satisfied with high probability.
For this, we first establish that satisfying the constraint function implies satisfaction of the immediate constraint function. 
\begin{lem}
\label{lem:relation_constr_funcs}
For any $x\in\mathcal{X}$, if there exists $a\in\mathcal{A}$ such that $g_i(a, x)\ge 0\,\forall i\in\mathcal{I}_\mathrm{g}$, then $\bar{g}_i(x)\ge 0\,\forall i\in\mathcal{I}_\mathrm{g}$.
\end{lem}
\begin{proof}
By definition of the immediate constraint function, it holds that
\begin{align*}
    \bar{g}_i(x) &\ge \min_{x'\in\traj{0}{x}{a}\bar{g}_i(x)} \bar{g}_i(x)\\
    &= g_i(a,x)\\
    &\ge 0
\end{align*}
for all $i\in\mathcal{I}_\mathrm{g}$.
\end{proof}
This lets us conclude that during experiments for which we stay inside the safe set, the immediate constraint function is satisfied with high probability.
\begin{lem}
\label{lem:safety_safe_set_mu}
For all $x\in\mathcal{X}$ such that $\exists a\in\mathcal{A}:(a,[x]_\mu)\in S_n$, we have with probability at least $1-\delta$, $\bar{g}_i(x)\ge 0$ for all $i\in\mathcal{I}_\mathrm{g}$. 
\end{lem}
\begin{proof}
By Assumption~\ref{ass:safe_set}(ii), for all $(a,[x]_\mu)\in S_n$ we have some $a'\in\mathcal{A}$ such that $g_i(a',[x]_\mu)\ge L_\mathrm{x}\mu$ for all $i\in\mathcal{I}_\mathrm{g}$ with probability at least $1-\delta$.
Thus, we have for all $i\in\mathcal{I}_\mathrm{g}$ that $g_i(a', [x]_\mu)\ge 0$ by Lipschitz-continuity and $\bar{g}_i(x) \ge 0$ by Lemma~\ref{lem:relation_constr_funcs}.
\end{proof}
Lastly, we need to prove that switching to a safe backup policy when evaluating unsafe policy parameters guarantees safety.
\begin{lem}
\label{lem:safety_S3}
With probability at least $1-\delta$, we have $\bar{g}_i(x(t))\ge 0\,\forall i\in\mathcal{I}_\mathrm{g},\forall t >0$ during experiments in \textbf{S3}.
\end{lem}
\begin{proof}
In \textbf{S3}, we start from $\x0$ such that $\exists a\in\mathcal{A}: (a,[x(t)]_\mu)\in S_n$.
For all $t$ such that $\exists a^t: (a^t, [x(t)]_\mu)\in S_n$, $\bar{g}_i(x(t))\ge 0\,\forall i\in\mathcal{I}_\mathrm{g}$ with probability at least $1-\delta$ follows from Lemma~\ref{lem:safety_safe_set_mu}.
If $\exists t'$ such that $[x(t')]_\mu\in\partial S_n$, we apply $a'$ such that $(a',[x(t')]_\mu)\in S_n$, where $g_i(a',[x(t')]_\mu)\ge L_\mathrm{x}\mu$ by Assumption~\ref{ass:safe_set}(iii).
Then, we have for all $i\in\mathcal{I}_\mathrm{g}$
\begin{align*}
    0&\le g_i(a',x(t'))\tag*{by Lipschitz continuity.}\\
    &= \min_{x'\in \xi_{(t',x(t'),a)}} \bar{g}_i(x')\tag*{Markov property of~\eqref{eqn:sysdyn}}.
\end{align*}
This proves that the second part of the trajectory after applying $a'$ is safe.
Safety of the first part is guaranteed by Lemmas~\ref{lem:safety_switching} and~\ref{lem:safety_safe_set_mu}.
\end{proof}
Combining the results presented so far then proves safety of \textbf{S3}.
\begin{lem}
\label{lem:safety_S3_update}
Assume a successful experiment of \textbf{S3}. 
For the evaluated $(a,\tilde{x}_0)$, we have that $g_i(a,\tilde{x}_0)\ge 0\,\forall i\in\mathcal{I}_\mathrm{g}$ with probability at least $1-\delta$.
\end{lem}
\begin{proof}
For a successful experiment \textbf{S3} we never have $[x(t)]_\mu \in \partial S_n$. Therefore, we have by Lemma~\ref{lem:safety_switching} and Assumption~\ref{ass:safe_set} that for all $t$, $\exists a':(a',[x(t)]_\mu)\in S_n \land g_i(a',[x(t)]_\mu)\ge L_\mathrm{x}\mu$ with probability at least $1-\delta$.
Thus, for each of these pairs $(a',[x(t)]_\mu)$, we have for all $i\in\mathcal{I}_\mathrm{g}$ with probability at least $1-\delta$, $g_i(a',x(t))\ge 0$ by Lipschitz-continuity.
\end{proof}
Thus, we have shown that during global exploration, we will not violate any safety constraints with high probability assuming a safe set $S_n$ as above.
We now show that updating the safe set preserves safety guarantees.
\begin{lem}
\label{lem:safety_update}
The safe set $S_n$ satisfies Assumption~\ref{ass:safe_set} for any $n\ge 0$.
\end{lem}
\begin{proof}
We prove the lemma by induction.
For the initial safe set at $n=0$, we assumed that for all $(a,\x0)\in S_0$ also states that are $\mu$-close are safe, \ie $g_i(a,\tilde{x}_0)\ge L_\mathrm{x}\mu\,\forall i\in\mathcal{I}_\mathrm{g}$.
Thus, for $n=0$ the claim holds by assumption.

For the induction step, we distinguish between local and global search.
For the local search, \ie \textbf{S1} and \textbf{S2}, the proof is a straightforward extension of~\cite[Lem.~11]{berkenkamp2016bayesian}.
In particular, assume some $n\ge 1$ such that Assumption~\ref{ass:safe_set} holds.
Then, if we do an update step \textbf{S1} or \textbf{S2}, we have for all $(a,\tilde{x}_0)\in S_n\setminus S_{n-1}, i\in\mathcal{I}_\mathrm{g}$, $\exists (a',\tilde{x}_0')\in S_{n-1}$ such that
\begin{align*}
    0&\le l_n(a',\tilde{x}_0',i)-L_\mathrm{a}\norm{a-a'}-L_\mathrm{x}(\norm{\tilde{x}_0-\tilde{x}_0'}+\mu)\\
    &\le g_i(a',\tilde{x}_0')-L_\mathrm{a}\norm{a-a'}-L_\mathrm{x}(\norm{\tilde{x}_0-\tilde{x}_0'}+\mu) \tag*{By~Cor.~\ref{cor:uncertainty_bound}}\\
    &\le g_i(a,\tilde{x}_0)-L_\mathrm{x}\mu.\tag*{By Lipschitz-continuity}
\end{align*}
Thus, all three parts of Assumption~\ref{ass:safe_set} are satisfied for $(a,\x0)$ pairs that are added in \textbf{S1} or \textbf{S2}.

Now, for \textbf{S3}, again assume some $n\ge1$ such that Assumption~\ref{ass:safe_set} holds for $S_{n-1}$.
Doing an update step \textbf{S3}, we evaluate a pair 
\begin{align}
\label{eqn:s3_update_safety}
(a,\tilde{x}_0)\notin\mathcal{E_\mathrm{f}}\cup S_{n-1}: a\in \mathcal{A} \land \exists a': (a', \tilde{x}_0) \in S_{n-1}.
\end{align}
For every successful experiment after which we update the safe set, we further have 
\begin{align}
\label{eqn:S3_boundary}
\forall x\in\traj{t_0}{\x0}{a}:x\notin\partial S_{n-1}.
\end{align}
The safe set is, in case of a successful experiment, updated with $S_n= S_{n-1}\cup (a,\x0)$.
For $(a,\x0)$, Lemma~\ref{lem:safety_S3_update} guarantees $g_i(a,\tilde{x}_0)\ge 0\,\forall i\in\mathcal{I}_\mathrm{g}$ with probability at least $1-\delta$, \ie Assumption~\ref{ass:safe_set}(i) is satisfied.
Further, due to~\eqref{eqn:s3_update_safety} Assumption~\ref{ass:safe_set}(ii) was already satisfied for the considered $\x0$ in $S_{n-1}$.
Finally,~\eqref{eqn:S3_boundary} guarantees $\x0\notin\partial S_{n-1}$.
Therefore, the update will not violate Assumption~\ref{ass:safe_set}(iii).
%
\end{proof}

To conclude the proof, we still need to prove safety during \textbf{S1} and \textbf{S2}.
\begin{lem}
\label{lem:safety_safe_set}
Let $S_n\subseteq \mathcal{A}\times\mathcal{X}_\mu$.
Then, with probability at least $1-\delta$, $\bar{g}_i(x)\geq 0$ for all $i\in\mathcal{I}_\mathrm{g},x\in \xi_{(t_0,\tilde{x}_0,a)}$, with $(a,\tilde{x}_0) \in S_n$, for all $t_0\ge 0$. 
\end{lem}
\begin{proof}
Assume to the contrary that $\exists x\in \xi_{(t_0,\tilde{x}_0,a)}$, with $(a,\tilde{x}_0) \in S_n$,
such that $\bar{g}_i(x) < 0$.
Then
\begin{align*}
    0&>\bar{g}_i(x) \\
    &\ge \min_{x'\in \xi_{(t_0,\tilde{x}_0,a)}} \bar{g}_i(x') \tag*{Since $x\in \xi_{(t_0,\tilde{x}_0,a)}$}\\
    &= \min_{x'\in \xi_{(0,\tilde{x}_0,a)}} \bar{g}_i(x') \tag*{Markov property of~\eqref{eqn:sysdyn}}\\
    &= g_i(a, \tilde{x}_0) \tag*{By def.}.
\end{align*}
This is a contradiction.
\end{proof}
Taking Lemma~\ref{lem:safety_S3} and Lemma~\ref{lem:safety_update} guarantees safety during \textbf{S3} while Lemma~\ref{lem:safety_safe_set} guarantees safety during \textbf{S1} and \textbf{S2}. Combining these results then proves the theorem.

\subsection{Extended Proof of Theorem~\ref{thm:glob_optimality} and Corollary~\ref{cor:glob_optimality}}

Having shown that \ourmethod provides safety guarantees, we now discuss optimality.
In particular, we will show that \ourmethod convergences, with $\epsilon$-precision and probability at least $1-\delta$, to the optimum within the maximum safely reachable set $\bar{R}_\epsilon$.
To prove $\epsilon$-convergence to the optimum within $\bar{R}_\epsilon$, we first need to establish that we sufficiently explore $\Rglob$.
We start by establishing that the safe set does not shrink.
\begin{lem}
\label{lem:monotonicity1}
For any $n\ge 1$ it holds that $S_{n+1}\supseteq S_n\supseteq S_0$.
\end{lem}
\begin{proof}
The update rules of the safe set are given by~\eqref{eqn:safeset_update} and~\eqref{eqn:update_safe_set_global}.
In both cases, we obtain $S_n$ by unifying newly found safe pairs $(a_n, \tilde{x}_n)$ with $S_{n-1}$.
Thus, $S_n\supseteq S_{n-1}$ holds for any $n$.
\end{proof}
We next show that if $S\subseteq R$, then also the closure of the reachability operator applied to both sets will have the same property.
To prove this for $\Rglob$, it is essential that experiments that are successful under $S$ are also successful under $\Rglob$.
Therefore, we first show that initial conditions that are in the projection of $S$ onto the state space, but not part of $\partial S$, cannot be part of $\partial R$.
\begin{lem}
\label{lem:border_R_vs_S}
If $S\subseteq R$, then $\partial R\subseteq\partial S\cup\{x\in\mathcal{X}_\mu\text{ such that }\nexists a\in\mathcal{A}:(x, a)\in S\}$.
\end{lem}
\begin{proof}
For the sake of contradiction, assume the opposite is true.
That is, assume $\exists x\in\partial R$ such that \emph{(i)} $\exists a\in\mathcal{A}:(x, a)\in S\subseteq R$ and \emph{(ii)} $x\notin\partial S$.
However, from \emph{(i)} and \emph{(ii)} it follows by definition that $\forall x'\in\mathcal{X}\text{ such that } \norm{x-x'}<2\mu\,~\exists a'\in\mathcal{A}:([x']_\mu,a')\in S\subseteq R$.
Thus, $x\notin\partial R$, which is a contradiction.
\end{proof}
This lets us conclude that successful experiments under $S$, \ie experiments during which we did not hit $\partial S$, will also be successful under $R$.
\begin{lem}
\label{lem:succ_exp_R_S}
If $S\subseteq R$ and $\nexists x\in\traj{0}{\x0}{a}$ such that $[x]_\mu\in\partial S$, with $\x0$ such that $\exists a'\in\mathcal{A}:(\x0, a')\in S$, then $\nexists x\in\traj{0}{\x0}{a}$ such that $[x]_\mu\in\partial R$.
\end{lem}
\begin{proof}
For the sake of contradiction, assume the opposite to be true.
That is, assume $\exists x\in\traj{0}{\x0}{a}$ such that $[x]_\mu\in\partial R$.
By Lemma~\ref{lem:border_R_vs_S}, we have that $\partial R\subseteq \partial S \cup \{x\in\mathcal{X}_\mu\text{ such that }\nexists a\in\mathcal{A}:(x, a)\in S\}$.
By hypothesis, we know that $[x]_\mu\notin\partial S$.
Thus, we would need to have $[x]_\mu\in \{x\in\mathcal{X}_\mu\text{ such that }\nexists a\in\mathcal{A}:(x, a)\in S\}$.
Using the same continuity argument as in Lemma~\ref{lem:safety_switching}, we can state that if $\exists t\text{ such that }x(t)\in\traj{0}{\x0}{a}$ and $\nexists a\in\mathcal{A}\text{ such that }(a, [x]_\mu)\in S$, then $\exists t'\le t\text{ such that }x(t')\in\traj{0}{\x0}{a}\cap\partial S$, which is a contradiction.
\end{proof}
With this, we can finally prove the desired statement.
\begin{lem}
\label{lem:monotonicity2}
It holds that $S\subseteq R\implies \bar{R}(S)\subseteq \bar{R}(R)$ for both $R=\Rglob$ and $R=\Rcon$, where the closure $\bar{R}_\epsilon^\mathrm{c}$ is defined in the same way as $\bar{R}_\epsilon$.
\end{lem}
\begin{proof}
Consider first $\Rcon$, the adapted reachability operator from~\cite{srinivas2012information}.
Let $(a,\x0)\in \Rcon(S)$.
Then, the proof is a straightforward extension of~\cite[Lem.~3(vi)]{berkenkamp2016bayesian}.
In particular, by definition, we have for all $i\in\mathcal{I}_\mathrm{g}$, $\exists(a',\x0')\in S, g_i(a',\x0')-\epsilon -L_\mathrm{a}\norm{a-a'}-L_\mathrm{x}(\norm{\x0-\x0'}+\mu)\ge 0$.
But, since $S\subseteq R$, we have that $(a',\x0')\in R$, which then also implies that $(a,\x0)\in \Rcon (R)$.
By repeated application, we get the statement for the closure.

Consider now the global reachability operator $\Rglob$.
Let $(a,\x0)\in \Rglob(S)\setminus\Rcon(S)$.
Then, by definition,  we have $\nexists x\in\traj{t_0}{\x0}{a}: [x]_\mu\in\partial \Rcon(S)$, which by Lemma~\ref{lem:succ_exp_R_S} implies $\nexists x\in\traj{t_0}{\x0}{a}: [x]_\mu\in\partial \Rcon(R)$ since $S\subseteq R$.
Thus, we have $\Rglob(S)\subseteq\Rglob(R)$.
Also here we get the statement about the closure by repeated application.
\end{proof}
This lets us conclude that if more points can be safely explored, the reachability operator $\Rglob$ will be non-empty.
\begin{lem}
\label{lem:converged}
For any $n\ge 1$, if $\bar{R}_\epsilon(S_0)\setminus S_n\neq\varnothing$, then $\Rglob(S_n)\setminus S_n\neq\varnothing$.
Similarly, $\bar{R}_\epsilon^\mathrm{c}(S_0)\setminus S_n\neq\varnothing$ implies $\Rcon(S_n)\setminus S_n\neq\varnothing$.
\end{lem}
\begin{proof}
Lemmas~\ref{lem:monotonicity1} and~\ref{lem:monotonicity2} satisfy the conditions for~\cite[Lem.~6]{berkenkamp2016bayesian} and the proof is the same as shown therein.
\end{proof}
In Alg.~\ref{alg:global_safeopt}, we state that we switch to a new stage if $w_n(\cdot)<\epsilon$.
Thus, to prevent our algorithm from getting stuck, we need to show that this happens eventually.
The statement is essentially the same as in~\cite[Cor.~2]{berkenkamp2016bayesian}.
As shown therein, the time step after which the uncertainty is below $\epsilon$ can be computed as the smallest integer $N_n$ satisfying $\frac{N_n}{\beta_{n+\mathrm{N}_n}\gamma_{\abs{\mathcal{I}}(n+\mathrm{N}_n)}}\ge\frac{C_1}{\epsilon^2}$ and depends on the noise in the system through $C_1=\frac{8}{\log(1+\sigma^{-2})}$, the information capacity $\gamma$, $\beta$, which we defined in Lemma~\ref{lem:uncertainty_bound}, and on the tolerance $\epsilon$.
Generally, we would need to define two individual $N_n$ for \textbf{S1} and \textbf{S2}, since \textbf{S1} acts in a reduced search space.
Consequently, the information capacity of \textbf{S2} should be larger.
Therefore, also its $N_n$ should be bigger.
For ease of presentation, we consider only one single $N_n$, which is the larger of both quantities, \ie the one we need in \textbf{S2}. 
\begin{lem}
\label{lem:uncertainty}
After a finite $N_n$ for which $S_n=S_{n+\mathrm{N}_n}$, we have $w_{n+\mathrm{N}_n}(a,\x0,i)<\epsilon\,~\forall i\in\mathcal{I}_\mathrm{g}$ for all $(a,\x0)\in G_{n+\mathrm{N}_n}$.
Further, we have $w_{n+\mathrm{N}_n}(a,x_0,i)<\epsilon\,~\forall i\in\mathcal{I}$ for all $(a,x_0)\in G_{n+\mathrm{N}_n}\cup M_{n+\mathrm{N}_n}$.
\end{lem}
\begin{proof}
In \textbf{S1}, we do standard \safeopt with $\x0=x_0$.
Thus, the proof for $(a,x_0)\in G_{n+\mathrm{N}_n}\cup M_{n+\mathrm{N}_n}$ follows from~\cite[Cor.~2]{berkenkamp2016bayesian}.

For \textbf{S2}, due to $S_n = S_{n+\mathrm{N}_n}$, we have $G_{n+\mathrm{N}_n}\subseteq G_n$ since for any $(a,\x0)\in S_n$, $e_n(a,\x0,i)$ is decreasing in $n$ for all $i\in\mathcal{I}_\mathrm{g}$, since $u_n(a,\x0,i)$ is decreasing in $n$.
By definition, we have $(a_n,\tilde{x}_{0,n})=\argmax_{(a,\x0)\in G_n,i\in\mathcal{I}_\mathrm{g}}w_n(a,\x0,i)$.
Thus, we have the same setting as~\cite[Cor.~2]{berkenkamp2016bayesian} in only $G$ instead of $G\cup M$.
Thus, the proof follows as shown therein.
\end{proof}
Taking these results, we can show that after a finite number of evaluations, the safe set must increase unless it is impossible to do so.
We first show that if there are pairs $(a,\x0)$ that can be explored during \textbf{S1} and \textbf{S2}, these will eventually be explored with high probability.
\begin{lem}
\label{lem:explore_s12}
For any $n\ge 1$, if $\bar{R}_\epsilon^\mathrm{c}(S_n)\setminus S_n\neq\varnothing$, then, with probability at least $1-\delta$,
$S_{n+\mathrm{N}_n}\supsetneq S_n$.
\end{lem}
\begin{proof}
The proof is a straightforward extension of~\cite[Lem.~7]{berkenkamp2016bayesian}.
In particular, by Lemma~\ref{lem:converged}, we get that $\Rcon(S_n)\setminus S_n\neq \varnothing$. Consider $(a,\tilde{x}_0) \in \Rcon(S_n)\setminus S_n$
Then, by definition, for all $i\in\mathcal{I}_g$
\begin{align}
\label{eqn:expansion1}
    &\exists (a', \tilde{x}_0')\in S_n:\\
    &g_i(a',\tilde{x}_0')-\epsilon-L_\mathrm{a}\norm{a-a'}-L_\mathrm{x}(\norm{\tilde{x}_0-\tilde{x}_0'}+\mu)\ge 0.\nonumber
\end{align}
Now, assume that after exploring $\mathrm{N}_n$ points, the safe set has not grown, i.e., $S_{n+\mathrm{N}_n}=S_n$.
Then, we have $(a, \tilde{x}_0)\in\mathcal{A}\times\mathcal{X}_\mu\setminus S_{n+\mathrm{N}_n}$ while $(a',\tilde{x}_0')\in S_{n+\mathrm{N}_n}$ and~\eqref{eqn:expansion1} holds.
Therefore, since by definition $u_n(\cdot)\ge g_i(\cdot)-\epsilon$ (\cf. Corollary~\ref{cor:uncertainty_bound}), we also have that $e_{n+\mathrm{N}_n}(a',\tilde{x}_0')>0$, which implies $(a',\tilde{x}_0')\in G_{n+\mathrm{N}_n}$.

Finally, as we have $S_{n+\mathrm{N}_n}=S_n$ and $(a',\tilde{x}_0')\in G_{n+\mathrm{N}_n}$,  Lemma~\ref{lem:uncertainty} implies
\begin{equation}
    \label{eqn:bound_w}
    w_{n+\mathrm{N}_n}(a',\tilde{x}_0',i)\le \epsilon\,\forall i\in\mathcal{I}_g.
\end{equation}
To keep notation uncluttered, define $\kappa\coloneqq L_\mathrm{a}\norm{a-a'}-L_\mathrm{x}(\norm{\tilde{x}_0-\tilde{x}_0'}+\mu)$.
Then, we have for all $i\in\mathcal{I}_g$
\begin{align*}
\label{eqn:expansion2}
    l_{n+\mathrm{N}_n}&(a', \tilde{x}_0',i)-\kappa\\\
    \ge &g_i(a',\tilde{x}_0')- w_{n+\mathrm{N}_n}(a',\tilde{x}_0',i)-\kappa \tag*{Cor.~\ref{cor:uncertainty_bound}}\\
    \ge &g_i(a',\tilde{x}_0')-\epsilon-\kappa \tag*{By~\eqref{eqn:bound_w}}\\
    \ge &0 \tag*{By~\eqref{eqn:expansion1}}.
\end{align*}
Thus, we get $(a,\tilde{x}_0)\in S_{n+\mathrm{N}_n}$, which is a contradiction.
\end{proof}
Now, we show the same for \textbf{S3}.
\begin{lem}
\label{lem:explore_s3}
For any $n\ge 1$, if $\bar{R}_\epsilon\setminus S_n\neq\varnothing$ and $\bar{R}_\epsilon^\mathrm{c}\setminus S_n =\varnothing$, then, with probability at least $1-\delta$,
$S_{n+\mathrm{N}_n+\abs{\mathcal{A}\times\mathcal{X}_\mu\setminus S_n}}\supsetneq S_n$.
\end{lem}
\begin{proof}
Lemma~\ref{lem:uncertainty} implies that the conditions for \textbf{S1} and \textbf{S2} are not satisfied and, thus, we are in $\textbf{S3}$.
In the general case, we have $\mathcal{E}_\mathrm{f}=\varnothing$ and~\eqref{eqn:global_init_next_point} recommends the most uncertain point in
\[
\{(a,\x0)\in(\mathcal{A}\times\mathcal{X}_\mu)\setminus (S_n\cup\mathcal{E}_\mathrm{f}): (a',\x0)\in S_n,a'\in\mathcal{A}\}.
\]
This set contains at most $\abs{\mathcal{A}\times\mathcal{X}_\mu\setminus S_n}$ points.
Denote with $(\bar{a},\bar{x}_0)$ the point recommended by~\eqref{eqn:global_init_next_point}.
After completing the rollout procedure in Alg.~\ref{alg:global_safeopt}, we have two possible outcomes:
\begin{subequations}
\begin{align}
    \label{eqn:s3_success}
    &\nexists x\in\traj{0}{\bar{x}_0}{\bar{a}}: x\in\partial S_n\\
    \label{eqn:s3_fail}
    &\exists x\in\traj{0}{\bar{x}_0}{\bar{a}}: x\in\partial S_n.
\end{align}
\end{subequations}
In case~\eqref{eqn:s3_success}, we update the safe set with~\eqref{eqn:update_safe_set_global} and $S_{n+\mathrm{N}_{n+1}}\supsetneq S_n$.
In case~\eqref{eqn:s3_fail}, we claim that $(\bar{a},\bar{x}_0)\notin\Rglob(S_n)$.
In particular, by assumption, we have $S_n=\Rcon(S_n)$.
However, $(\bar{a},\bar{x}_0)\in\mathcal{A}\times\mathcal{X}_\mu\setminus(S_n\times\mathcal{E}_\mathrm{f})$, and, hence, $(\bar{a},\bar{x}_0)\notin\Rcon(S_n)$.
Moreover,~\eqref{eqn:s3_fail} implies that $\exists x\in\traj{0}{\bar{x}_0}{\bar{a}}: x\in\partial \Rcon(S_n)$ as $\Rcon(S_n)=S_n$.
As $\Rglob$ is the union of 2 sets and the reasoning above forbids $(\bar{a},\bar{x}_0)$ to be part of either, we have $(\bar{a},\bar{x}_0)\notin \Rglob(S_n)$.
Thus, $(\bar{a},\bar{x}_0)$ is added to $\mathcal{E}_\mathrm{f}$ and not suggested in further iterations.

\textbf{S3} can be repeated at most $\abs{\mathcal{A}\times\mathcal{X}_\mu\setminus S_n}$ times.
If we are in case~\eqref{eqn:s3_fail} for the first $\abs{\mathcal{A}\times\mathcal{X}_\mu\setminus S_n}-1$ iterations, we end up with $\abs{\mathcal{A}\times\mathcal{X}_\mu\setminus (S_n\cup\mathcal{E}_\mathrm{f})}=1$.
Since, by assumption, $\bar{R}_\epsilon\setminus S_n\neq\varnothing$, we have $\Rglob\neq\varnothing$ by Lemma~\ref{lem:converged}.
Thus, at iteration $\abs{\mathcal{A}\times\mathcal{X}_\mu\setminus S_n}$, we must have $(\bar{a},\bar{x}_0)\in\Rglob(S_n)$.
By definition of $\Rglob(S_n)$, we know that the trajectory induced by $(\bar{a},\bar{x}_0)$ will not hit $\partial \Rcon(S_n)$.
Further, as we have $\Rcon(S_n)\setminus S_n=\varnothing$ by assumption, we have $\Rcon(S_n) = S_n$ and. therefore, $\partial\Rcon(S_n)=\partial S_n$.
That is, $(\bar{a},\bar{x}_0)$ will induce a successful experiment.
Then, we can update the safe set with~\eqref{eqn:update_safe_set_global}, which completes the proof.
\end{proof}
Taking both results lets us finally state that we will eventually enlarge the safe set unless it is impossible to do so:
\begin{cor}
\label{cor:explore}
If $\bar{R}_\epsilon\setminus S_n\neq\varnothing$, $S_{n+\mathrm{N}_n+\abs{\mathcal{A}\times\mathcal{X}_\mu\setminus S_n}}\supsetneq S_n$.
\end{cor}
\begin{proof}
Due to Lemma~\ref{lem:converged}, $\bar{R}_\epsilon\setminus S_n\neq\varnothing$ implies $\Rglob\setminus S_n\neq\varnothing$.
Thus, the conditions for either \textbf{S1} or \textbf{S2} or \textbf{S3} are satisfied.
In case \textbf{S1} or \textbf{S2}, $S_{n+\mathrm{N}_n}\supsetneq S_n$ follows from Lemma~\ref{lem:explore_s12}.
In case \textbf{S3}, $S_{n+\mathrm{N}_n+\abs{\mathcal{A}\times\mathcal{X}_\mu\setminus S_n}}\supsetneq S_n$ follows from Lemma~\ref{lem:explore_s3}.
\end{proof}
\begin{remark}
In Alg.~\ref{alg:global_safeopt} and the theoretical results, we assumed that we first expand locally (using \textbf{S1} and \textbf{S2}) as much as possible and only then do a global search step \textbf{S3}.
However, the definition of the global reachability operator $\Rglob$ and its closure also allow alternating the stages.
Following such an approach would result in many more \textbf{S3} experiments that would need to be stopped prematurely and repeated as the safe set increases, but, ultimately, would converge to the same result.
\end{remark}
We can then show that we get $\epsilon-$close to the optimum value within the safely reachable set, $\bar{R}_\epsilon(S_0)$:
\begin{lem}
\label{lem:eps_convergence}
Define $\hat{a} = \argmax_{(a, \tilde{x}_0=x_0)\in S_n}l_n(a, x_0,0)$, where setting the last argument of $l_n(\cdot)$ to zero indicates that we are considering the lower confidence bound of the reward function.
Further, define $N_{\mathrm{max},n}=\abs{\mathcal{A}\times\mathcal{X}_\mu\setminus S_n}+N_n$.
If for any $n\ge1$, $S_{n+\mathrm{N}{_\mathrm{max},n}}=S_n$, then the following holds with probability at least $1-\delta$:
\begin{align}
    f(\hat{a},x_0) \ge \max_{(a,x_0)\in\bar{R}_\epsilon(S_0)}f(a,x_0)-\epsilon.
\end{align}
\end{lem}
\begin{proof}
Let 
\[
(a^*,x_0) \coloneqq \argmax_{(a,x_0)\in S_{n+\mathrm{N}_{\mathrm{max},n}}}f(a,x_0),
\]
which is by definition contained in $M_{n+\mathrm{N}_{\mathrm{max},n}}$.
This implies $w_{n+\mathrm{N}_{\mathrm{max},n}}(a^*,x_0,i) < \epsilon\,\forall i\in\mathcal{I}$ by Lemma~\ref{lem:uncertainty}.
For the sake of contradiction, assume that 
\begin{align}
\label{eqn:opt_contr}
f(\hat{a},x_0) < f(a^*,x_0)-\epsilon.
\end{align}
Then, we have
\begin{align*}
    l_{n+\mathrm{N}_{\mathrm{max},n}}(a^*,x_0,0)&\le l_{n+\mathrm{N}_{\mathrm{max},n}}(\hat{a},x_0,0) \tag*{By Def.\ of $\hat{a}$}\\
    &\le f(\hat{a},x_0)\tag*{Lem.~\ref{lem:uncertainty_bound}}\\
    &< f(a^*,x_0)-\epsilon\tag*{By~\eqref{eqn:opt_contr}}\\
    &\le u_{n+\mathrm{N}_{\mathrm{max},n}}(a^*,x_0,0)-\epsilon\tag*{Lem.~\ref{lem:uncertainty_bound}}\\
    &\le l_{n+\mathrm{N}_{\mathrm{max},n}}(a^*,x_0,0), \tag*{as $a^*\in M_{n+\mathrm{N}_{\mathrm{max},n}}$}
\end{align*}
which is a contradiction.

Finally, by Corollary~\ref{cor:explore}, $S_{n+\mathrm{N}_{\mathrm{max},n}}=S_n$ implies that $\bar{R}_\epsilon(S_0)\subseteq S_n=S_{n+\mathrm{N}_{\mathrm{max},n}}$.
Therefore,
\begin{align}
    \max_{(a,x_0)\in\bar{R}_\epsilon(S_0)}f(a,x_0)-\epsilon &\le \max_{(a,x_0)\in S_{n+\mathrm{N}_{\mathrm{max},n}}}f(a,x_0)-\epsilon\nonumber\\
    &=f(a^*,x_0)-\epsilon\\
    &\le f(\hat{a},x_0).\nonumber
\end{align}
\end{proof}
For Lemma~\ref{lem:eps_convergence}, we need that $S_{n+\mathrm{N}{_\mathrm{max},n}}=S_n$ is satisfied for some $n$.
The fact that this happens follows from~\cite[Lem.~10]{berkenkamp2016bayesian}.
Thus, \ourmethod will converge to the optimum of the safely reachable set with $\epsilon$-precision with probability at least $1-\delta$, which proves Theorem~\ref{thm:glob_optimality}.
If the optimum of the safely reachable set is at the same time the global optimum, \ourmethod converges with $\epsilon$-precision to the global optimum, which proves Corollary~\ref{cor:glob_optimality}.

\end{document}